\documentclass[letterpaper]{article} % DO NOT CHANGE THIS
\usepackage{aaai24}  % DO NOT CHANGE THIS
\usepackage{times}  % DO NOT CHANGE THIS
\usepackage{helvet}  % DO NOT CHANGE THIS
\usepackage{courier}  % DO NOT CHANGE THIS
\usepackage{xcolor}
\usepackage{amsmath}
\usepackage{multirow}
\usepackage{subfigure}
\usepackage{tikz}
\usepackage{amsfonts}
\usepackage{booktabs}
\usepackage{amssymb}
\usepackage[hyphens]{url}  % DO NOT CHANGE THIS
\usepackage{graphicx} % DO NOT CHANGE THIS
\usepackage{natbib}  % DO NOT CHANGE THIS AND DO NOT ADD ANY OPTIONS TO IT
\usepackage{caption} % DO NOT CHANGE THIS AND DO NOT ADD ANY OPTIONS TO IT
\usepackage{algorithm}
\usepackage{algorithmic}
\usepackage{amsthm}
\newtheorem{definition}{Definition}
\newtheorem{theorem}{Theorem}
\usepackage{amsmath}
\usepackage{amssymb}

\usepackage{newfloat}
\usepackage{listings}
\DeclareCaptionStyle{ruled}{labelfont=normalfont,labelsep=colon,strut=off} % DO NOT CHANGE THIS
\floatstyle{ruled}
\newfloat{listing}{tb}{lst}{}
\floatname{listing}{Listing}
\title{MSGNet: Learning Multi-Scale Inter-Series Correlations for Multivariate Time Series Forecasting}
\author{
    Wanlin Cai\textsuperscript{\rm 1},
    Yuxuan Liang\textsuperscript{\rm 2},
    Xianggen Liu\textsuperscript{\rm 1},
    Jianshuai Feng\textsuperscript{\rm 3},
    Yuankai Wu\textsuperscript{\rm 1 \footnote{Corresponding author}}
}

\affiliations{
    \textsuperscript{\rm 1} Sichuan University \\ \textsuperscript{\rm 2} The Hong Kong University of Science and Technology (Guangzhou) \\
    \textsuperscript{\rm 3} Beijing Institute of Technology\\
    yozhiboo@gmail.com, yuxliang@outlook.com, liuxgcs@foxmail.com, 3120185238@bit.edu.cn, wuyk0@scu.edu.cn
}

\usepackage{bibentry}
\begin{document}

\maketitle

\begin{abstract}
Multivariate time series forecasting poses an ongoing challenge across various disciplines. Time series data often exhibit diverse intra-series and inter-series correlations, contributing to intricate and interwoven dependencies that have been the focus of numerous studies. Nevertheless, a significant research gap remains in comprehending the varying inter-series correlations across different time scales among multiple time series, an area that has received limited attention in the literature. To bridge this gap, this paper introduces MSGNet, an advanced deep learning model designed to capture the varying inter-series correlations across multiple time scales using frequency domain analysis and adaptive graph convolution. By leveraging frequency domain analysis, MSGNet effectively extracts salient periodic patterns and decomposes the time series into distinct time scales. The model incorporates a self-attention mechanism to capture intra-series dependencies, while introducing an adaptive mixhop graph convolution layer to autonomously learn diverse inter-series correlations within each time scale. Extensive experiments are conducted on several real-world datasets to showcase the effectiveness of MSGNet. Furthermore, MSGNet possesses the ability to automatically learn explainable multi-scale inter-series correlations, exhibiting strong generalization capabilities even when applied to out-of-distribution samples. Code is available at \url{https://github.com/YoZhibo/MSGNet}.

\end{abstract}

\section{Introduction}

Throughout centuries, the art of forecasting has been an invaluable tool for scientists, policymakers, actuaries, and salespeople. Its foundation lies in recognizing that hidden outcomes, whether in the future or concealed, often reveal patterns from past observations. Forecasting involves skillfully analyzing available data, unveiling interdependencies and temporal trends to navigate uncharted territories with confidence and envision yet-to-be-encountered instances with clarity and foresight. In this context, time series forecasting emerges as a fundamental concept, enabling the analysis and prediction of data points collected over time, offering insights into variables like stock prices~\cite{cao2022ai}, weather conditions~\cite{bi2023accurate}, or customer behavior~\cite{salinas2020deepar}. 

Two interconnected realms within time series forecasting come into play: \emph{intra-series correlation} modeling, which predicts future values based on patterns within a specific time series, and \emph{inter-series correlation} modeling, which explores relationships and dependencies between multiple time series. Recently, deep learning models have emerged as a catalyst for breakthroughs in time series forecasting. On one hand, Recurrent Neural Networks (RNNs)~\cite{salinas2020deepar}, Temporal Convolution Networks (TCNs)~\cite{yue2022ts2vec}, and Transformers~\cite{zhou2021informer} have demonstrated exceptional potential in capturing temporal dynamics within individual series. Simultaneously, a novel perspective arises when considering multivariate time series as graph signals. In this view, the variables within a multivariate time series can be interpreted as nodes within a graph, interconnected through hidden dependency relationships. Consequently, Graph Neural Networks (GNNs)~\cite{kipf2016semi} offer a promising avenue for harnessing the intricate interdependencies among multiple time series. %压缩版本

\begin{figure}[t]

\centerline{\includegraphics[scale=0.36]{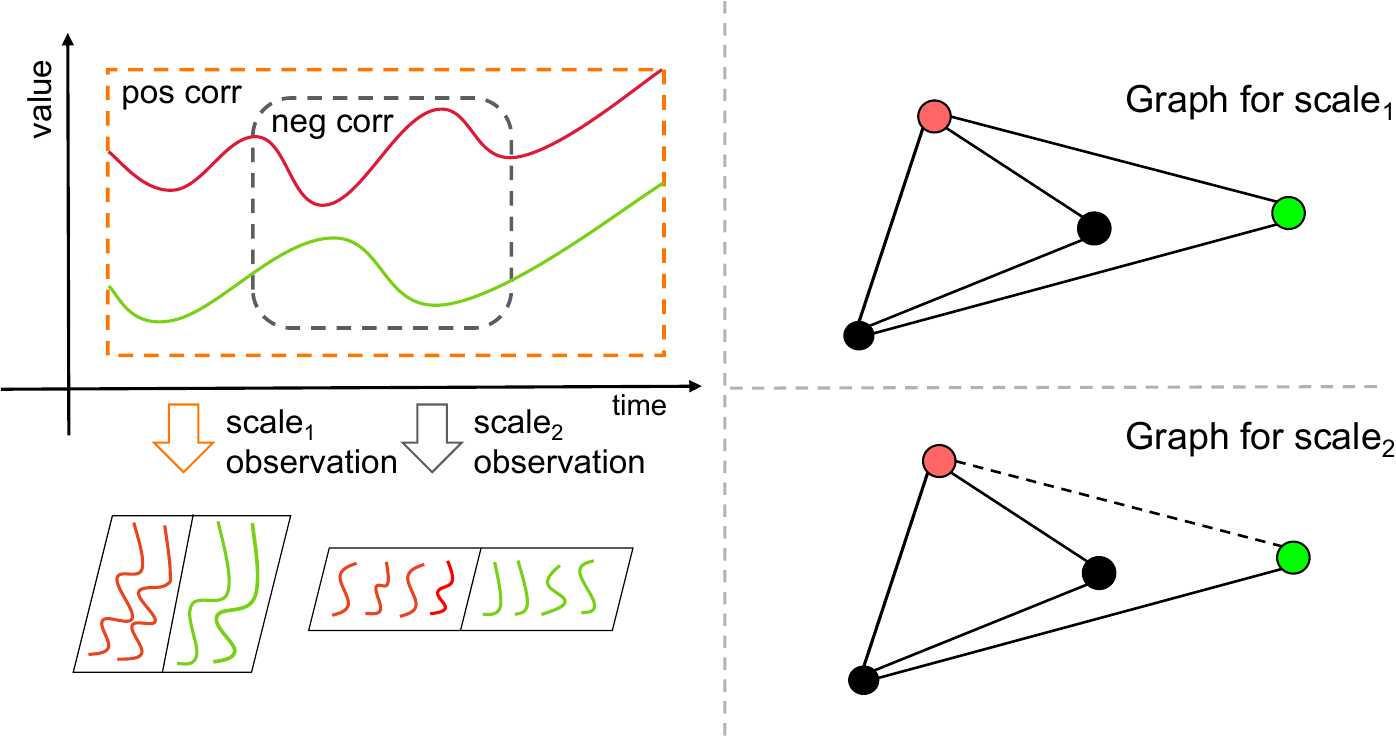}}
\caption{In the longer time $\text{scale}_1$, the green and red time series are positively correlated, whereas in the shorter time $\text{scale}_2$, they exhibit a negative correlation. Consequently, we observe two distinct graph structures in these two different time scales.}
\label{fig: concept}
\end{figure}

Within the domain of time series analysis, there is a significant oversight concerning the varying inter-series correlations across different time scales among multiple time series, which the existing deep learning models fail to accurately describe. For instance, in the realm of finance, the correlations among various asset prices, encompassing stocks, bonds, and commodities, during periods of market instability, asset correlations may increase due to a flight-to-safety phenomenon. Conversely, during economic growth, asset correlations might decrease as investors diversify their portfolios to exploit various opportunities~\cite{baele2020flights}. Similarly, in ecological systems, the dynamics governing species populations and environmental variables reveal intricate temporal correlations operating at multiple time scales~\cite{whittaker2001scale}. In Figure~\ref{fig: concept}, we provide an example where, in time $\text{scale}_1$, we can observe a positive correlation between two time series, whereas in the shorter $\text{scale}_2$, we might notice a negative correlation between them. By employing the graph-based approach, we obtain two distinct graph structures.

In the aforementioned examples, the limitation of existing deep learning models becomes apparent, as they often fail to capture the diverse interdependencies and time-varying correlations between the variables in consideration. For instance, when relying solely on one type of inter-series correlation, such as utilizing GNNs with one fixed graph structure~\cite{yu2018spatio, li2018diffusion}, these models may suffer from diminished predictive accuracy and suboptimal forecasting performance in scenarios characterized by intricate and varying inter-series correlations. While some methods consider using dynamic and time-varying graph structures to model inter-series correlations~\cite{zheng2020gman, guo2021learning}, they overlook the crucial fact that these correlations may be intimately tied to time scales of notable stability, exemplified by economic and environmental cycles.

Addressing the identified gaps and aiming to overcome the limitations of prior models, we introduce MSGNet, which is comprised of three essential components: the scale learning and transforming layer, the multiple graph convolution module, and the temporal multi-head attention module. Recognizing the paramount importance of periodicity in time series data and to capture dominant time scales effectively, we leverage the widely recognized Fast Fourier transformation (FFT) method. By applying FFT to the original time series data, we project it into spaces linked to the most prominent time scales. This approach enables us to aptly capture and represent various inter-series correlations unfolding at distinct time scales. Moreover, we introduce a multiple adaptive graph convolution module enriched with a learnable adjacency matrix. For each time scale, a dedicated adjacency matrix is dynamically learned. Our framework further incorporates a multi-head self-attention mechanism adept at capturing intra-series temporal patterns within the data. Our contributions are summarized in three folds:
\begin{itemize}
    \item We make a key observation that inter-series correlations are intricately associated with different time scales. To address this, we propose a novel structure named MSGNet that efficiently discovers and captures these multi-scale inter-series correlations. 
    \item To tackle the challenge of capturing both intra-series and inter-series correlations simultaneously, we introduce a combination of multi-head attention and adaptive graph convolution modules. 
    \item Through extensive experimentation on real-world datasets, we provide empirical evidence that MSGNet consistently outperforms existing deep learning models in time series forecasting tasks. Moreover, MSGNet exhibits better generalization capability.
\end{itemize}

\section{Related Works}

\subsection{Time Series Forecasting}

Time series forecasting has a long history, with classical methods like VAR~\cite{kilian2017structural} and Prophet~\cite{taylor2018forecasting} assuming that intra-series variations follow pre-defined patterns. However, real-world time series often exhibit complex variations that go beyond the scope of these pre-defined patterns, limiting the practical applicability of classical methods. In response, recent years have witnessed the emergence of various deep learning models, including MLPs~\cite{oreshkin2019n, zeng2023transformers}, TCNs~\cite{yue2022ts2vec}, RNNs~\cite{rangapuram2018deep, gasthaus2019probabilistic, salinas2020deepar} and Transformer-based models~\cite{zhou2021informer,wu2021autoformer, zhou2022fedformer, wen2022transformers, wang2023micn}, designed for time series analysis. Yet, an ongoing question persists regarding the most suitable candidate for modeling intra-series correlations, whether it be MLP or transformer-based architectures~\cite{Yuqietal-2023-PatchTST, das2023long}. Some approaches have considered periodicities as crucial features in time series analysis. For instance, DEPTS~\cite{fan2022depts} instantiates periodic functions as a series of cosine functions, while TimesNet~\cite{wu2023timesnet} performs periodic-dimensional transformations of sequences. Notably, none of these methods, though, give consideration to the diverse inter-series correlations present at different periodicity scales, which is a central focus of this paper.

\begin{figure*}[htbp]
\centering
\subfigure{\includegraphics[scale=0.35]{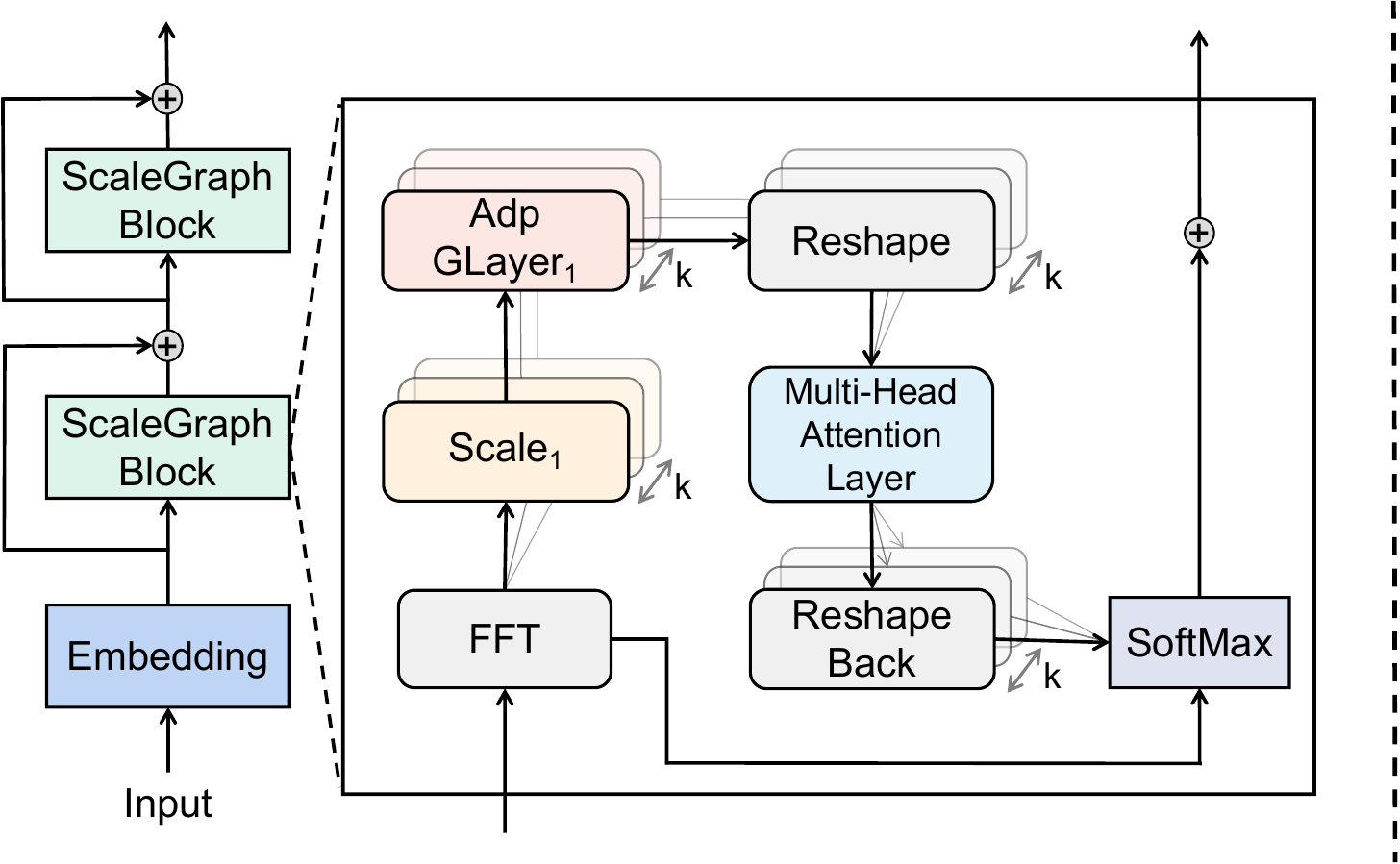}}
\subfigure
{\includegraphics[scale=0.36]{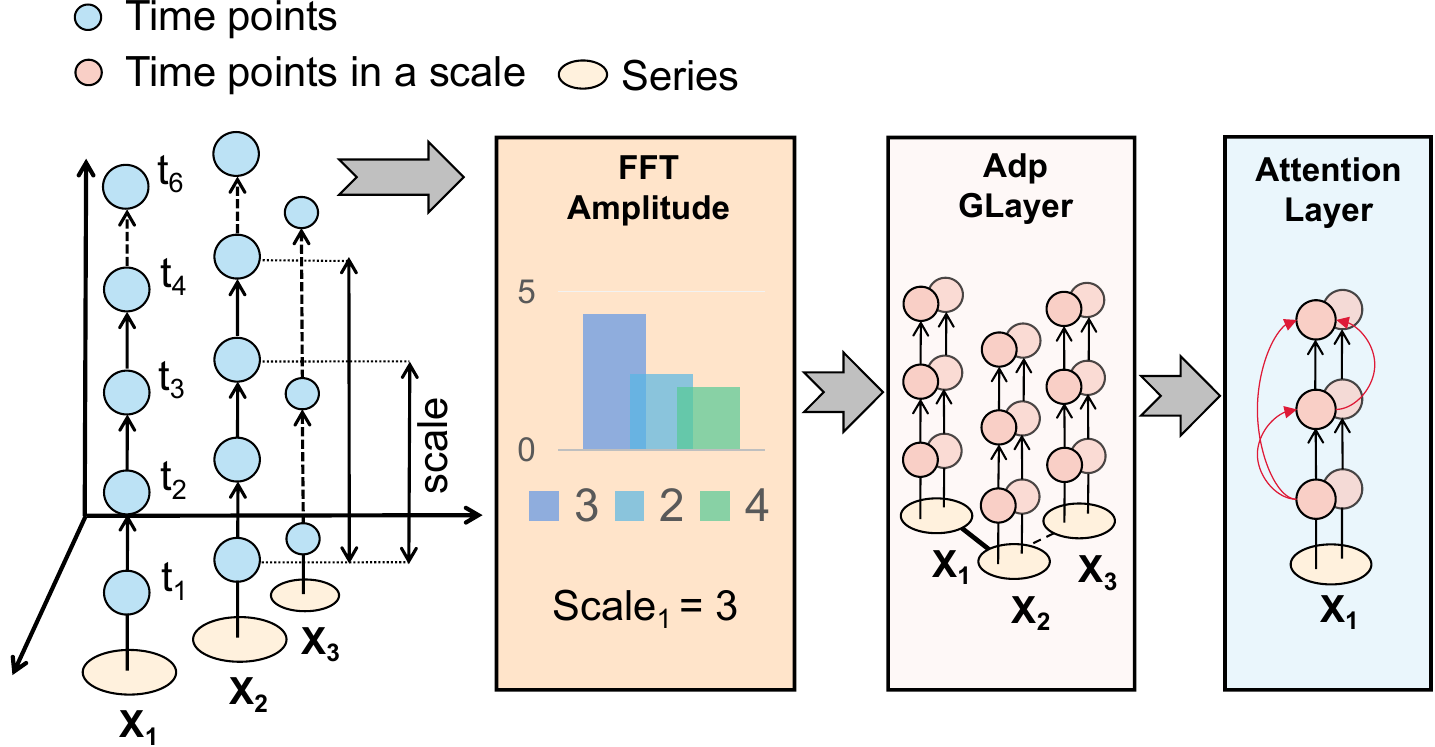}}
\caption{ MSGNet employs several ScaleGraph blocks, each encompassing three pivotal modules: an FFT module for multi-scale data identification, an adaptive graph convolution module for inter-series correlation learning within a time scale, and a multi-head attention module for intra-series correlation learning. }
 \label{fig: model}
\end{figure*}

\subsection{GNNs for Inter-series Correlation Learning}

Recently, there has been a notable rise in the use of GNNs~\cite{defferrard2016convolutional, kipf2016semi, abu2019mixhop} for learning inter-series correlations. Initially introduced to address traffic prediction~\cite{li2018diffusion, yu2018spatio, cini2023scalable, wu2023spatiotemporal} and skeleton-based action recognition~\cite{shi2019skeleton}, GNNs have demonstrated significant improvements over traditional methods in short-term time series prediction. However, it is important to note that most existing GNNs are designed for scenarios where a pre-defined graph structure is available. For instance, in traffic prediction, the distances between different sensors can be utilized to define the graph structure. Nonetheless, when dealing with general multivariate forecasting tasks, defining a general graph structure based on prior knowledge can be challenging. Although some methods have explored the use of learnable graph structures~\cite{wu2019graph, bai2020adaptive, wu2020connecting}, they often consider a limited number of graph structures and do not connect the learned graph structures with different time scales. Consequently, these approaches may not fully capture the intricate and evolving inter-series correlations. 

% \begin{figure*}[htb!]
% \centerline{\includegraphics[scale=0.18]{model.pdf}}
% \caption{MSGNet model structure diagram: It comprises three key modules: an FFT module for transforming data into multiple scales, an adaptive graph convolution module for capturing inter-series correlations within one time scale, and a multi-head attention module for capturing intra-series correlations. }
% \label{fig: model}
% \end{figure*}

\section{Problem Formulation} % 表达出时间序列在不同scale下具有多种关联结构

In the context of multivariate time series forecasting, consider a scenario where the number of variables is denoted by $N$. We are given input data $\mathbf{X}_{t-L: t} \in \mathbb{R}^{N \times L}$, which represents a retrospective window of observations, comprising ${X}^{i}_{\tau}$ values at the $\tau\text{th}$ time point for each variable $i$ in the range from $t-L$ to $t-1$. Here, $L$ represents the size of the retrospective window, and $t$ denotes the initial position of the forecast window. The objective of the time series forecasting task is to predict the future values of the $N$ variables for a time span of $T$ future time steps. The predicted values are represented by $\mathbf{\hat{X}}_{t: t+T} \in \mathbb{R}^{N \times T}$, which includes ${X}^{i}_{\tau}$ values at each time point $\tau$ from $t$ to $t+T-1$ for all the variables.

We assume the ability to discern varying inter-series correlations among $N$ time series at different time scales, which can be represented by graphs. For instance, given a time scale $s_i< L$, we can identify a graph structure $\mathcal{G}_i = \{ \mathcal{V}_i, \mathcal{E}_i\}$ from the time series $\mathbf{X}_{p-s_i: p}$. Here, $\mathcal{V}_i$ denotes a set of nodes with $|\mathcal{V}_i| = N$, $\mathcal{E}_i \subseteq \mathcal{V}_i \times \mathcal{V}_i$ represents the weighted edges, and $p$ denotes an arbitrary time point. Considering a collection of $k$ time scales, denoted as $\{s_1, \cdots, s_k\}$, we can identify $k$ adjacency matrices, represented as $\{\mathbf{A}^1, \cdots, \mathbf{A}^k\}$, where each $\mathbf{A}^k \in \mathbb{R}^{N \times N}$. These adjacency matrices capture the varying inter-series correlations at different time scales. 

\section{Methodology} %两个原则，多图卷积和具有强表征能力的图卷积

As previously mentioned, our work aims to bridge the gaps in existing time series forecasting models through the introduction of MSGNet, a novel framework designed to capture diverse inter-series correlations at different time scales. The overall model architecture is illustrated in Figure~\ref{fig: model}. Comprising multiple ScaleGraph blocks, MSGNet's essence lies in its ability to seamlessly intertwine various components. Each ScaleGraph block entails a four-step sequence: 1) Identifying the scales of input time series; 2) Unveiling scale-linked inter-series correlations using adaptive graph convolution blocks; 3) Capturing intra-series correlations through multi-head attention; and 4) Adaptively aggregating representations from different scales using a SoftMax function. 

% \begin{figure}[htb!]
% \centerline{\includegraphics[width=8cm, height=5.25cm]{model.pdf}}
% \caption{MSGNet model structure diagram. }
% \label{fig: model}
% \end{figure}

\subsection{Input Embedding and Residual Connection}

We embed $N$ variables at the same time step into a vector of size $d_{\text{model}}$: $\mathbf{X}_{t-L: t} \to \mathbf{X}_{\text{emb}}$, where $\mathbf{X}_{\text{emb}} \in \mathbb{R}^{d_{\text{model}} \times L}$. We employ the uniform input representation proposed in \cite{zhou2021informer} to generate the embedding. Specifically, $\mathbf{X}_{\text{emb}}$ is calculated using the following equation:
\begin{equation}
\mathbf{X}_{\text{emb}} = \alpha \text{Conv1D}(\mathbf{\hat{X}}_{t-L: t}) + \mathbf{PE} + \sum^P_{p = 1} \mathbf{SE}_p.
\end{equation}
Here, we first normalize the input $\mathbf{X}_{t-L:t}$ and obtain $\mathbf{\hat{X}}_{t-L:t}$, as the normalization strategy has been proven effective in improving stationarity~\cite{liu2022non}. Then we project $\mathbf{\hat{X}}_{t-L: t}$ into a $d_{\text{model}}$-dimensional matrix using 1-D convolutional filters (kernel width=3, stride=1). The parameter $\alpha$ serves as a balancing factor, adjusting the magnitude between the scalar projection and the local/global embeddings. $\mathbf{PE} \in \mathbb{R}^{d_{\text{model}} \times L}$ represents the positional embedding of input $\mathbf{X}$, and $\mathbf{SE}_p \in \mathbb{R}^{d_{\text{model}} \times L}$ is a learnable global time stamp embedding with a limited vocabulary size (60 when minutes as the finest granularity).

We implement MSGNet in a residual manner \cite{he2016deep}. At the very outset, we set $\mathbf{X}^{0} = \mathbf{X}_{\text{emb}}$, where $\mathbf{X}_{\text{emb}}$ represents the raw inputs projected into deep features by the embedding layer. In the $l$-th layer of MSGNet, the input is $\mathbf{X}^{l-1} \in \mathbb{R}^{d_{\text{model}} \times L}$, and the process can be formally expressed as follows:
\begin{equation}
\mathbf{X}^{l}=\text{ScaleGraphBlock}\left(\mathbf{X}^{l-1}\right)+\mathbf{X}^{l-1},
\end{equation}
Here, $\text{ScaleGraphBlock}$ denotes the operations and computations that constitute the core functionality of the MSGNet layer.

\subsection{Scale Identification}

Our objective is to enhance forecasting accuracy by leveraging inter-series correlations at different time scales. The choice of scale is a crucial aspect of our approach, and we place particular importance on selecting periodicity as the scale source. The rationale behind this choice lies in the inherent significance of periodicity in time series data. For instance, in the daytime when solar panels are exposed to sunlight, the time series of energy consumption and solar panel output tend to exhibit a stronger correlation. This correlation pattern would differ if we were to choose a different periodicity, such as considering the correlation over the course of a month or a day.

Inspired by TimesNet~\cite{wu2023timesnet}, we employ the Fast Fourier Transform (FFT) to detect the prominent periodicity as the time scale:
\begin{equation}
\begin{split}
\mathbf{F} = \text{Avg} \left( \text{Amp}  \left(\text{FFT}  (\mathbf{X}_{\text{emb}})\right) \right)), \\
f_1, \cdots, f_k  = \underset{f_* \in \{1, \cdots, \frac{L}{2} \} } {\text{argTopk}} (\mathbf{F}), 
s_i =  \frac{L}{f_i}.
\end{split}
\end{equation}
Here, $\text{FFT}(\cdot)$ and $\text{Amp}(\cdot)$ denote the FFT and the calculation of amplitude values, respectively. The vector $\mathbf{F} \in \mathbb{R}^L$ represents the calculated amplitude of each frequency, which is averaged across $d_{\text{model}}$ dimensions by the function $\text{Avg}(\cdot)$. 

In this context, it is noteworthy that the temporally varying inputs may demonstrate distinct periodicities, thereby allowing our model to detect evolving scales. We posit that the correlations intrinsic to this time-evolving periodic scale remain stable. This viewpoint leads us to observe dynamic attributes in the inter-series and intra-series correlations learned by our model.

Based on the selected time scales \(\{s_1, \ldots, s_k\}\), we can get several representations corresponding to different time scales by reshaping the inputs into 3D tensors using the following equations:
\begin{equation}
\mathbf{\mathcal{X}}^i = \text{Reshape}_{s_i, f_i}(\text{Padding}(\mathbf{X}_\text{in})), \quad i \in \{1, \ldots, k\},
\end{equation}
where \(\text{Padding}(\cdot)\) is used to extend the time series by zeros along the temporal dimension to make it compatible for \(\text{Reshape}_{s_i, f_i}(\cdot)\). Note that \(\mathbf{\mathcal{X}}^i \in \mathbb{R}^{d_\text{model} \times s_i \times f_i }\) denotes the \(i\)-th reshaped time series based on time scale \(i\). We use $\mathbf{X}_\text{in}$ to denote the input matrix of the ScaleGraph block.

\subsection{Multi-scale Adaptive Graph Convolution} %需要再讨论

We propose a novel multi-scale graph convolution approach to capture specific and comprehensive inter-series dependencies. To achieve this, we initiate the process by projecting the tensor corresponding to the $i$-th scale back into a tensor with $N$ variables, where $N$ represents the number of time series. This projection is carried out through a linear transformation, defined as follows:
\begin{equation}
\mathbf{\mathcal{H}}^i = \mathbf{{W}}^i \mathbf{\mathcal{X}}^i.
\end{equation}
Here, $\mathbf{\mathcal{H}}^i \in \mathbb{R}^{N \times s_i \times f_i}$,  and $\mathbf{{W}}^i \in \mathbb{R}^{N \times d_\text{model}}$ is a learnable weight matrix, tailored to the $i$-th scale tensor. One may raise concerns that inter-series correlation could be compromised following the application of linear mapping and subsequent linear mapping back. However, our comprehensive experiments demonstrate a noteworthy outcome: the proposed approach adeptly preserves the inter-series correlation by the graph convolution approach.

The graph learning process in our approach involves generating two trainable parameters, $\mathbf{E}^i_1$ and $\mathbf{E}^i_2 \in \mathbb{R}^{N \times h}$. Subsequently, an adaptive adjacency matrix is obtained by multiplying these two parameter matrices, following the formula:
\begin{equation}
\mathbf{A}^i = \text{SoftMax}(\text{ReLu}(\mathbf{E}^i_1 (\mathbf{E}^i_2)^T)).
\end{equation}
In this formulation, we utilize the $\text{SoftMax}$ function to normalize the weights between different nodes, ensuring a well-balanced and meaningful representation of inter-series relationships. 

%\begin{equation}
%{H}_{\text {scale}_{{i}}}^{\text {out }}=\left(\Vert {H}_{\text {scale }_{{i}}}^{({k})}\right) {W}_{\text {scale }_{{i}}}, {k} \in[0, {~K}],
%\end{equation}

%\begin{figure}[t]
%\centerline{\includegraphics[width=6.8cm, height=5cm]{mix.pdf}}
%\caption{The two green modules are the input and output parts respectively, and the blue module connected to the input is the mixed propagation process and intermediate variables. The updated intermediate variables and adjacency matrix will be used in each iteration. The red module concatenates intermediate variables to retain more information. Finally, a simple MLP is used to learn the trainable weights and obtain the output, which occurs at each scale.}
%\label{fig: mix}
%\end{figure}

After obtaining the adjacency matrix $\mathbf{A}^i$ for the $i$-th scale, we utilize the Mixhop graph convolution method~\cite{abu2019mixhop} to capture the inter-series correlation, as its proven capability to represent features that other models may fail to capture (See Appendix). The graph convolution is defined as follows:
\begin{equation}
\mathbf{\mathcal{H}}^{{{i}}}_{\text {out }} = \sigma\left(\underset{j \in \mathcal{P}}{\Vert} (\mathbf{A}^i)^j \mathbf{\mathcal{H}}^{{{i}}}\right),
\label{eq:ag}
\end{equation}
where $\mathbf{\mathcal{H}}^{{{i}}}_{\text {out }}$ represents the output after fusion at scale $i$, $\sigma()$ is the activation function, the hyper-parameter P is a set of integer adjacency
powers, $(\mathbf{A}^i)^j$ denotes the learned adjacency matrix $\mathbf{A}^i$ multiplied by itself $j$ times, and $\Vert$ denotes a column-level connection, linking intermediate variables generated during each iteration. Then, we proceed to utilize a multi-layer perceptron (MLP) to project $\mathbf{\mathcal{H}}^{{{i}}}_{\text {out }}$ back into a 3D tensor $\mathbf{\mathcal{\hat{X}}}^i \in \mathbb{R}^{d_{\text{model}} \times s_i \times f_i}$.

\subsection{Multi-head Attention and Scale Aggregation}

In each time scale, we employ the Multi-head Attention (MHA) to capture the intra-series correlations. Specifically, for each time scale tensor $\mathbf{\mathcal{\hat{X}}}^i$, we apply self MHA on the time scale dimension of the tensor:
\begin{equation}
\mathbf{\mathcal{\hat{X}}}^i_\text{out} = \text{MHA}_s(\mathbf{\mathcal{\hat{X}}}^i).
\end{equation}
Here, $\text{MHA}_s(\cdot)$ refers to the multi-head attention function proposed in~\cite{vaswani2017attention} in the scale dimension. For implementation, it involves reshape the input tensor of size $B \times d_{\text{model}} \times s_i \times f_i$ into a $B f_i \times d_{\text{model}} \times s_i $ tensor, $B$ is the batch size. Although some studies have raised concerns about the effectiveness of MHA in capturing long-term temporal correlations in time series~\cite{zeng2023transformers}, we have successfully addressed this limitation by employing scale transformation to convert long time spans into periodic lengths. Our results, as presented in the Appendix, show that MSGNet maintains its performance consistently even as the input time increases. 

Finally, to proceed to the next layer, we need to integrate $k$ different scale tensors ${\mathbf{\mathcal{\hat{X}}}^1_\text{out}, \cdots, \mathbf{\mathcal{\hat{X}}}^k_\text{out} }$. We first reshape the tensor of each scale back to a 2-way matrix $\mathbf{{\hat{X}}}^i_\text{out} \in \mathbb{R}^{d_{\text{model}} \times L}$. Then, we aggregate the different scales based on their amplitudes:
\begin{equation}
\begin{split}
\hat{a}_1, \cdots, \hat{a}_k = & \text{SoftMax}(\mathbf{F}_{f_1}, \cdots, \mathbf{F}_{f_k}), \\
\mathbf{{\hat{X}}}_\text{out} = & \sum^k_{i =1} \hat{a}_i \mathbf{{\hat{X}}}^i_\text{out}.
\end{split}
\end{equation}
In this process, $\mathbf{F}_{f_1}, \cdots, \mathbf{F}_{f_k}$ are amplitudes corresponding to each scale, calculated using the FFT. The SoftMax function is then applied to compute the amplitudes $\hat{a}_1, \cdots, \hat{a}_k$. This Mixture of Expert (MoE)~\cite{jacobs1991adaptive} strategy enables the model to emphasize information from different scales based on their respective amplitudes, facilitating the effective incorporation of multi-scale features into the next layer (Appendix).

\subsection{Output Layer}

To perform forecasting, our model utilizes linear projections in both the time dimension and the variable dimension to transform $\mathbf{{\hat{X}}}_\text{out} \in \mathbb{R}^{d_{\text{model}} \times L}$ into $\mathbf{\hat{X}}_{t: t+T} \in \mathbb{R}^{N \times T}$. This transformation can be expressed as:
\begin{equation}
\mathbf{\hat{X}}_{t: t+T} = \mathbf{W_s} \mathbf{{\hat{X}}}_\text{out} \mathbf{W_t} + \mathbf{b}.
\end{equation}
Here, $\mathbf{W_s} \in \mathbb{R}^{N \times d_{\text{model}}}$, $\mathbf{W_t} \in \mathbb{R}^{L \times T}$, and $\mathbf{b} \in \mathbb{R}^T$ are learnable parameters. The $\mathbf{W_s}$ matrix performs the linear projection along the variable dimension, and $\mathbf{W_t}$ does the same along the time dimension. The resulting $\mathbf{\hat{X}}_{t: t+T}$ is the forecasted data, where $N$ represents the number of variables, $L$ denotes the input sequence length, and $T$ signifies the forecast horizon.

\begin{table*}[htb!]
\small
\tabcolsep=0.17cm
\renewcommand\arraystretch{0.95}
\centerline{
% \resizebox{1\linewidth}{!}{
\begin{tabular}{cccccccccccccccc}
\toprule
\multicolumn{2}{c}{{\color[HTML]{333333} Models}}     & \multicolumn{2}{c}{\textbf{Ours}}                                             & \multicolumn{2}{c}{TimesNet}                                                  & \multicolumn{2}{c}{DLinear}                                                   & \multicolumn{2}{c}{NLinear}                                                   & \multicolumn{2}{c}{MTGnn}                                                  & \multicolumn{2}{c}{Autoformer}                                                & \multicolumn{2}{c}{Informer}     \\ \midrule
\multicolumn{2}{c}{Metric}                                                    & MSE                                   & MAE                                   & MSE                                   & MAE                                   & MSE                                   & MAE                                   & MSE                                   & MAE                                   & MSE                                   & MAE                                & MSE                                   & MAE                                   & MSE           & MAE                        \\ %\hline
\toprule
\multicolumn{1}{c|}{}                              & \multicolumn{1}{c|}{96}  & {\color[HTML]{000000} \textbf{0.183}} & {\color[HTML]{000000} \textbf{0.301}} & 0.237                                 & 0.350                                 & 0.221                                 & 0.337                                 & 0.270                                 & 0.379                                 & {\color[HTML]{000000} {\underline{0.196}}}    & {\color[HTML]{000000} {\underline{0.316}}} & {\color[HTML]{000000} {0.204}}    & {\color[HTML]{000000} {0.319}}    & 0.333         & 0.405                    \\
\multicolumn{1}{c|}{}                              & \multicolumn{1}{c|}{192} & {\color[HTML]{000000} \textbf{0.189}} & {\color[HTML]{000000} \textbf{0.306}} & 0.224                                 & 0.337                                 & 0.220                                 & 0.336                                 & 0.272                                 & 0.380                                 & 0.272                                 & 0.379                              & {\color[HTML]{000000} {\underline{0.200}}}    & {\color[HTML]{000000} {\underline{0.314}}}    & 0.358         & 0.421                    \\
\multicolumn{1}{c|}{}                              & \multicolumn{1}{c|}{336} & {\color[HTML]{000000} {\underline{0.206}}}    & {\color[HTML]{000000} {\underline{0.320}}}    & 0.289                                 & 0.394                                 & 0.229                                 & 0.342                                 & 0.280                                 & 0.385                                 & 0.260                                 & 0.369                              & {\color[HTML]{000000} \textbf{0.201}} & {\color[HTML]{000000} \textbf{0.318}} & 0.398         & 0.446                   \\
\multicolumn{1}{c|}{\multirow{-4}{*}{{Flight}}}      & \multicolumn{1}{c|}{720} & {\color[HTML]{000000} \textbf{0.253}} & {\color[HTML]{000000} \textbf{0.358}} & 0.310                                 & 0.408                                 & {\color[HTML]{000000} {\underline{0.263}}}    & {\color[HTML]{000000} {\underline{0.366}}}    & 0.316                                 & 0.409                                 & 0.390                                 & 0.449                              & 0.345                                 & 0.426                                 & 0.476         & 0.484                    \\ \midrule
\multicolumn{1}{c|}{}                              & \multicolumn{1}{c|}{96}  & {\color[HTML]{000000} \textbf{0.163}} & {\color[HTML]{000000} \textbf{0.212}} & 0.172                                 & {\color[HTML]{000000} {\underline{0.220}}}    & 0.196                                 & 0.255                                 & 0.196                                 & 0.235                                 & {\color[HTML]{000000} {\underline{0.171}}}    & 0.231                              & 0.266                                 & 0.336                                 & 0.300         & 0.384                   \\
\multicolumn{1}{c|}{}                              & \multicolumn{1}{c|}{192} & {\color[HTML]{000000} \textbf{0.212}} & {\color[HTML]{000000} \textbf{0.254}} & 0.219                                 & {\color[HTML]{000000} {\underline{0.261}}}    & 0.237                                 & 0.296                                 & 0.241                                 & 0.271                                 & {\color[HTML]{000000} {\underline{0.215}}}    & 0.274                              & 0.307                                 & 0.367                                 & 0.598         & 0.544                    \\
\multicolumn{1}{c|}{}                              & \multicolumn{1}{c|}{336} & {\color[HTML]{000000} { \underline{0.272} }}    & {\color[HTML]{000000} \textbf{0.299}} & 0.280                                 & {\color[HTML]{000000} { \underline{0.306} }}    & 0.283                                 & 0.335                                 & 0.293                                 & 0.308                                 & {\color[HTML]{000000} \textbf{0.266}} & 0.313                              & 0.359                                 & 0.395                                 & 0.578         & 0.523                    \\
\multicolumn{1}{c|}{\multirow{-4}{*}{{Weather}}}     & \multicolumn{1}{c|}{720} & 0.350                                 & {\color[HTML]{000000} \textbf{0.348}} & 0.365                                 & 0.359                                 & {\color[HTML]{000000} {\underline{0.345}}}    & 0.381                                 & 0.366                                 & {\color[HTML]{000000} {\underline{0.356}}}    & {\color[HTML]{000000} \textbf{0.344}} & 0.375                              & 0.419                                 & 0.428                                 & 1.059         & 0.741                    \\ \midrule
\multicolumn{1}{c|}{}                              & \multicolumn{1}{c|}{96}  & {\color[HTML]{000000} \textbf{0.319}} & {\color[HTML]{000000} \textbf{0.366}} & {\color[HTML]{000000} {\underline{0.338}}}    & 0.375                                 & 0.345                                 & 0.372                                 & 0.350                                 & {\color[HTML]{000000} {\underline{0.371}}}    & 0.381                                 & 0.415                              & 0.505                                 & 0.475                                 & 0.672         & 0.571                    \\
\multicolumn{1}{c|}{}                              & \multicolumn{1}{c|}{192} & {\color[HTML]{000000} {\underline{0.376}}}    & 0.397                                 & {\color[HTML]{000000} \textbf{0.374}} & {\color[HTML]{000000} \textbf{0.387}} & 0.380                                 & {\color[HTML]{000000} {\underline{0.389}}}    & 0.389                                 & 0.390                                 & 0.442                                 & 0.451                              & 0.553                                 & 0.496                                 & 0.795         & 0.669                    \\
\multicolumn{1}{c|}{}                              & \multicolumn{1}{c|}{336} & 0.417                                 & 0.422                                 & {\color[HTML]{000000} \textbf{0.410}} & {\color[HTML]{000000} \textbf{0.411}} & {\color[HTML]{000000} {\underline{0.413}}}    & 0.413                                 & 0.422                                 & {\color[HTML]{000000} {\underline{0.412}}}    & 0.475                                 & 0.475                              & 0.621                                 & 0.537                                 & 1.212         & 0.871                    \\
\multicolumn{1}{c|}{\multirow{-4}{*}{{ETTm1}}}       & \multicolumn{1}{c|}{720} & 0.481                                 & 0.458                                 & {\color[HTML]{000000} {\underline{0.478}}}    & {\color[HTML]{000000} {\underline{0.450}}}    & {\color[HTML]{000000} \textbf{0.474}} & 0.453                                 & 0.482                                 & {\color[HTML]{000000} \textbf{0.446}} & 0.531                                 & 0.507                              & 0.671                                 & 0.561                                 & 1.166         & 0.823                    \\ \midrule
\multicolumn{1}{c|}{}                              & \multicolumn{1}{c|}{96}  & {\color[HTML]{000000} \textbf{0.177}} & {\color[HTML]{000000} \textbf{0.262}} & {\color[HTML]{000000} {\underline{0.187}}}    & {\color[HTML]{000000} {\underline{0.267}}}    & 0.193                                 & 0.292                                 & 0.188                                 & 0.272                                 & 0.240                                 & 0.343                              & 0.255                                 & 0.339                                 & 0.365         & 0.453                    \\
\multicolumn{1}{c|}{}                              & \multicolumn{1}{c|}{192} & {\color[HTML]{000000} \textbf{0.247}} & {\color[HTML]{000000} \textbf{0.307}} & {\color[HTML]{000000} {\underline{0.249}}}    & {\color[HTML]{000000} {\underline{0.309}}}    & 0.284                                 & 0.362                                 & 0.253                                 & 0.312                                 & 0.398                                 & 0.454                              & 0.281                                 & 0.340                                 & 0.533         & 0.563                    \\
\multicolumn{1}{c|}{}                              & \multicolumn{1}{c|}{336} & {\color[HTML]{000000} \textbf{0.312}} & {\color[HTML]{000000} \textbf{0.346}} & 0.321                                 & 0.351                                 & 0.369                                 & 0.427                                 & {\color[HTML]{000000} {\underline{0.314}}}    & {\color[HTML]{000000} {\underline{0.350}}}    & 0.568                                 & 0.555                              & 0.339                                 & 0.372                                 & 1.363         & 0.887                   \\
\multicolumn{1}{c|}{\multirow{-4}{*}{{ETTm2}}}       & \multicolumn{1}{c|}{720} & {\color[HTML]{000000} {\underline{0.414}}}    & {\color[HTML]{000000} \textbf{0.403}} & {\color[HTML]{000000} \textbf{0.408}} & {\color[HTML]{000000} {\underline{0.403}}}    & 0.554                                 & 0.522                                 & 0.414                                 & 0.405                                 & 1.072                                 & 0.767                              & 0.433                                 & 0.432                                 & 3.379         & 1.338                    \\ \midrule
\multicolumn{1}{c|}{}                              & \multicolumn{1}{c|}{96}  & 0.390                                 & 0.411                                 & {\color[HTML]{000000} \textbf{0.384}} & 0.402                                 & {\color[HTML]{000000} {\underline{0.386}}}    & {\color[HTML]{000000} \textbf{0.400}} & 0.393                                 & {\color[HTML]{000000} {\underline{0.400}}}    & 0.440                                 & 0.450                              & 0.449                                 & 0.459                                 & 0.865         & 0.713                    \\
\multicolumn{1}{c|}{}                              & \multicolumn{1}{c|}{192} & 0.442                                 & 0.442                                 & {\color[HTML]{000000} \textbf{0.436}} & {\color[HTML]{000000} \textbf{0.429}} & {\color[HTML]{000000} {\underline{0.437}}}    & {\color[HTML]{000000} {\underline{0.432}}}    & 0.449                                 & 0.433                                 & 0.449                                 & 0.433                              & 0.500                                 & 0.482                                 & 1.008         & 0.792                   \\
\multicolumn{1}{c|}{}                              & \multicolumn{1}{c|}{336} & {\color[HTML]{000000} \textbf{0.480}} & 0.468                                 & 0.491                                 & 0.469                                 & {\color[HTML]{000000} {\underline{0.481}}}    & {\color[HTML]{000000} {\underline{0.459}}}    & 0.485                                 & {\color[HTML]{000000} \textbf{0.448}} & 0.598                                 & 0.554                              & 0.521                                 & 0.496                                 & 1.107         & 0.809                    \\
\multicolumn{1}{c|}{\multirow{-4}{*}{{ETTh1}}}       & \multicolumn{1}{c|}{720} & {\color[HTML]{000000} {\underline{0.494}}}    & {\color[HTML]{000000} {\underline{0.488}}}    & 0.521                                 & 0.500                                 & 0.519                                 & 0.516                                 & {\color[HTML]{000000} \textbf{0.469}} & {\color[HTML]{000000} \textbf{0.461}} & 0.685                                 & 0.620                              & 0.514                                 & 0.512                                 & 1.181         & 0.865                    \\ \midrule
\multicolumn{1}{c|}{}                              & \multicolumn{1}{c|}{96}  & {\color[HTML]{000000} {\underline{0.328}}}    & {\color[HTML]{000000} {\underline{0.371}}}    & 0.340                                 & 0.374                                 & 0.333                                 & 0.387                                 & {\color[HTML]{000000} \textbf{0.322}} & {\color[HTML]{000000} \textbf{0.369}} & 0.496                                 & 0.509                              & 0.346                                 & 0.388                                 & 3.755         & 1.525                    \\
\multicolumn{1}{c|}{}                              & \multicolumn{1}{c|}{192} & {\color[HTML]{000000} \textbf{0.402}} & {\color[HTML]{000000} \textbf{0.414}} & {\color[HTML]{000000} {\underline{0.402}}}    & {\color[HTML]{000000} {\underline{0.414}}}    & 0.477                                 & 0.476                                 & 0.410                                 & 0.419                                 & 0.716                                 & 0.616                              & 0.456                                 & 0.452                                 & 5.602         & 1.931                    \\
\multicolumn{1}{c|}{}                              & \multicolumn{1}{c|}{336} & {\color[HTML]{000000} \textbf{0.435}} & {\color[HTML]{000000} \textbf{0.443}} & 0.452                                 & 0.452                                 & 0.594                                 & 0.541                                 & {\color[HTML]{000000} {\underline{0.444}}}    & {\color[HTML]{000000} {\underline{0.449}}}    & 0.718                                 & 0.614                              & 0.482                                 & 0.486                                 & 4.721         & 1.835                    \\
\multicolumn{1}{c|}{\multirow{-4}{*}{{ETTh2}}}       & \multicolumn{1}{c|}{720} & {\color[HTML]{000000} \textbf{0.417}} & {\color[HTML]{000000} \textbf{0.441}} & 0.462                                 & 0.468                                 & 0.831                                 & 0.657                                 & {\color[HTML]{000000} {\underline{0.450}}}    & {\color[HTML]{000000} {\underline{0.462}}}    & 1.161                                 & 0.791                              & 0.515                                 & 0.511                                 & 3.647         & 1.625                    \\ \midrule
\multicolumn{1}{c|}{}                              & \multicolumn{1}{c|}{96}  & {\color[HTML]{000000} \textbf{0.165}} & 0.274                                 & {\color[HTML]{000000} {\underline{0.168}}}    & {\color[HTML]{000000} \textbf{0.272}} & 0.197                                 & 0.282                                 & 0.198                                 & {\color[HTML]{000000} {\underline{0.274}}}    & 0.211                                 & 0.305                              & 0.201                                 & 0.317                                 & 0.274         & 0.368                    \\
\multicolumn{1}{c|}{}                              & \multicolumn{1}{c|}{192} & {\color[HTML]{000000} \textbf{0.184}} & 0.292                                 & {\color[HTML]{000000} {\underline{0.184}}}    & 0.289                                 & 0.196                                 & {\color[HTML]{000000} {\underline{0.285}}}    & 0.197                                 & {\color[HTML]{000000} \textbf{0.277}} & 0.225                                 & 0.319                              & 0.222                                 & 0.334                                 & 0.296         & 0.386                    \\
\multicolumn{1}{c|}{}                              & \multicolumn{1}{c|}{336} & {\color[HTML]{000000} \textbf{0.195}} & 0.302                                 & {\color[HTML]{000000} {\underline{0.198}}}    & {\color[HTML]{000000} {\underline{0.300}}}    & 0.209                                 & 0.301                                 & 0.211                                 & {\color[HTML]{000000} \textbf{0.292}} & 0.247                                 & 0.340                              & 0.231                                 & 0.338                                 & 0.300         & 0.394                    \\
\multicolumn{1}{c|}{\multirow{-4}{*}{{Electricity}}} & \multicolumn{1}{c|}{720} & {\color[HTML]{000000} {\underline{0.231}}}    & 0.332                                 & {\color[HTML]{000000} \textbf{0.220}} & {\color[HTML]{000000} \textbf{0.320}} & 0.245                                 & 0.333                                 & 0.253                                 & {\color[HTML]{000000} {\underline{0.325}}}    & 0.287                                 & 0.373                              & 0.254                                 & 0.361                                 & 0.373         & 0.439         \\ \midrule
\multicolumn{1}{c|}{}                              & \multicolumn{1}{c|}{96}  & 0.102                                 & 0.230                                 & 0.107                                 & 0.234                                 & {\color[HTML]{000000} {\underline{0.088}}}    & {\color[HTML]{000000} {\underline{0.218}}}    & {\color[HTML]{000000} \textbf{0.088}} & {\color[HTML]{000000} \textbf{0.205}} & 0.267                                 & 0.378                              & 0.197                                 & 0.323                                 & 0.847         & 0.752                    \\
\multicolumn{1}{c|}{}                              & \multicolumn{1}{c|}{192} & 0.195                                 & 0.317                                 & 0.226                                 & 0.344                                 & {\color[HTML]{000000} \textbf{0.176}} & {\color[HTML]{000000} {\underline{ 0.315} }}    & {\color[HTML]{000000} {\underline{0.177}}}    & {\color[HTML]{000000} \textbf{0.297}} & 0.590                                 & 0.578                              & 0.300                                 & 0.369                                 & 1.204         & 0.895                   \\
\multicolumn{1}{c|}{}                              & \multicolumn{1}{c|}{336} & 0.359                                 & 0.436                                 & 0.367                                 & 0.448                                 & {\color[HTML]{000000} \textbf{0.313}} & {\color[HTML]{000000} {\underline{0.427}}}    & {\color[HTML]{000000} {\underline{0.323}}}    & {\color[HTML]{000000} \textbf{0.409}} & 0.939                                 & 0.749                              & 0.509                                 & 0.524                                 & 1.672         & 1.036                    \\
\multicolumn{1}{c|}{\multirow{-4}{*}{{Exchange}}}    & \multicolumn{1}{c|}{720} & 0.940                                 & 0.738                                 & 0.964                                 & 0.746                                 & {\color[HTML]{000000}\textbf{0.839}} & {\color[HTML]{000000} \textbf{0.695}} & {\color[HTML]{000000} {\underline{0.923}}}    & {\color[HTML]{000000} {\underline{0.725}}}    & 1.107                                 & 0.834                              & 1.447                                 & 0.941                                 & 2.478         & 1.310                    \\ \toprule
\multicolumn{2}{c|}{Avg Rank}                                                 & \multicolumn{2}{c|}{{\color[HTML]{000000} \textbf{1.813}}}                    & \multicolumn{2}{c|}{{\color[HTML]{000000} {\underline{2.750}}}}                       & \multicolumn{2}{c|}{3.563}                                                    & \multicolumn{2}{c|}{{\color[HTML]{000000}2.813}}                             & \multicolumn{2}{c|}{5.313}                                                 & \multicolumn{2}{c|}{4.750}                                                    & \multicolumn{2}{c}{7.000}              \\ \bottomrule

\end{tabular}
% }
}
\caption{ Forecast results with 96 review window and prediction length $\{96,192,336,720 \}$. The best result is represented in bold, followed by underline.}
\label{tab: tab1}
\end{table*}

\section{Experiments}
\subsection{Datasets}
To evaluate the advanced capabilities of MSGNet in time series forecasting, we conducted experiments on 8 datasets, namely Flight, Weather, ETT (h1, h2, m1, m2)~\cite{zhou2021informer}, Exchange-Rate~\cite{lai2018modeling} and Electricity. With the exception of the Flight dataset, all these datasets are commonly used in existing literature. The Flight dataset's raw data is sourced from the OpenSky official website\footnote{https://opensky-network.org/}, and it includes flight data related to the COVID-19 pandemic. In Figure 1 and 2 of Appendix, we visualize the changes in flight data during this period. Notably, the flights were significantly affected by the pandemic, resulting in out-of-distribution (OOD) samples for all deep learning models. This provides us with an opportunity to assess the robustness of the proposed model against OOD samples.

% \subsection{Baselines}
% We have chosen seven time series forecasting methods for comparison, encompassing models such as Transformer~\cite{vaswani2017attention}, Informer~\cite{zhou2021informer}, and Autoformer~\cite{wu2021autoformer}, which are based on transformer architectures. Furthermore, we included MTGnn~\cite{wu2020connecting}, which relies on graph convolution, as well as DLinear and NLinear~\cite{zeng2023transformers}, which are linear models. Lastly, we considered TimesNet~\cite{wu2023timesnet}, which is based on periodic decomposition and currently holds the state-of-the-art performance.
\subsection{Baselines}
We have chosen six time series forecasting methods for comparison, encompassing models such as Informer~\cite{zhou2021informer}, and Autoformer~\cite{wu2021autoformer}, which are based on transformer architectures. Furthermore, we included MTGnn~\cite{wu2020connecting}, which relies on graph convolution, as well as DLinear and NLinear~\cite{zeng2023transformers}, which are linear models. Lastly, we considered TimesNet~\cite{wu2023timesnet}, which is based on periodic decomposition and currently holds the state-of-the-art performance.

\subsection{Experimental Setups}
The experiment was conducted using an NVIDIA GeForce RTX 3090 24GB GPU, with the Mean Squared Error (MSE) used as the training loss function. The review window size of all models was set to $L=96$ (for fair comparison), and the prediction lengths were $T=\left\{96,192,336,720\right\}$. It should be noted that our model can achieve better performance with longer review windows (see Appendix). These settings were applied to all models. The initial learning rate was $LR=0.0001$, batch size was $Batch=32$, and the number of epochs was $Epochs=10$,  and early termination was used where applicable. For more details on hyperparameter settings of our model, please refer to Appendix. (0.7, 0.1, 0.2) or (0.6, 0.2, 0.2) of the data are used as training, validation, and test data, respectively. As for baselines, relevant data from the papers~\cite{wu2023timesnet} or official code~\cite{wu2020connecting} was utilized.

\begin{figure}[htb!]
\centerline{\includegraphics[scale=0.21]{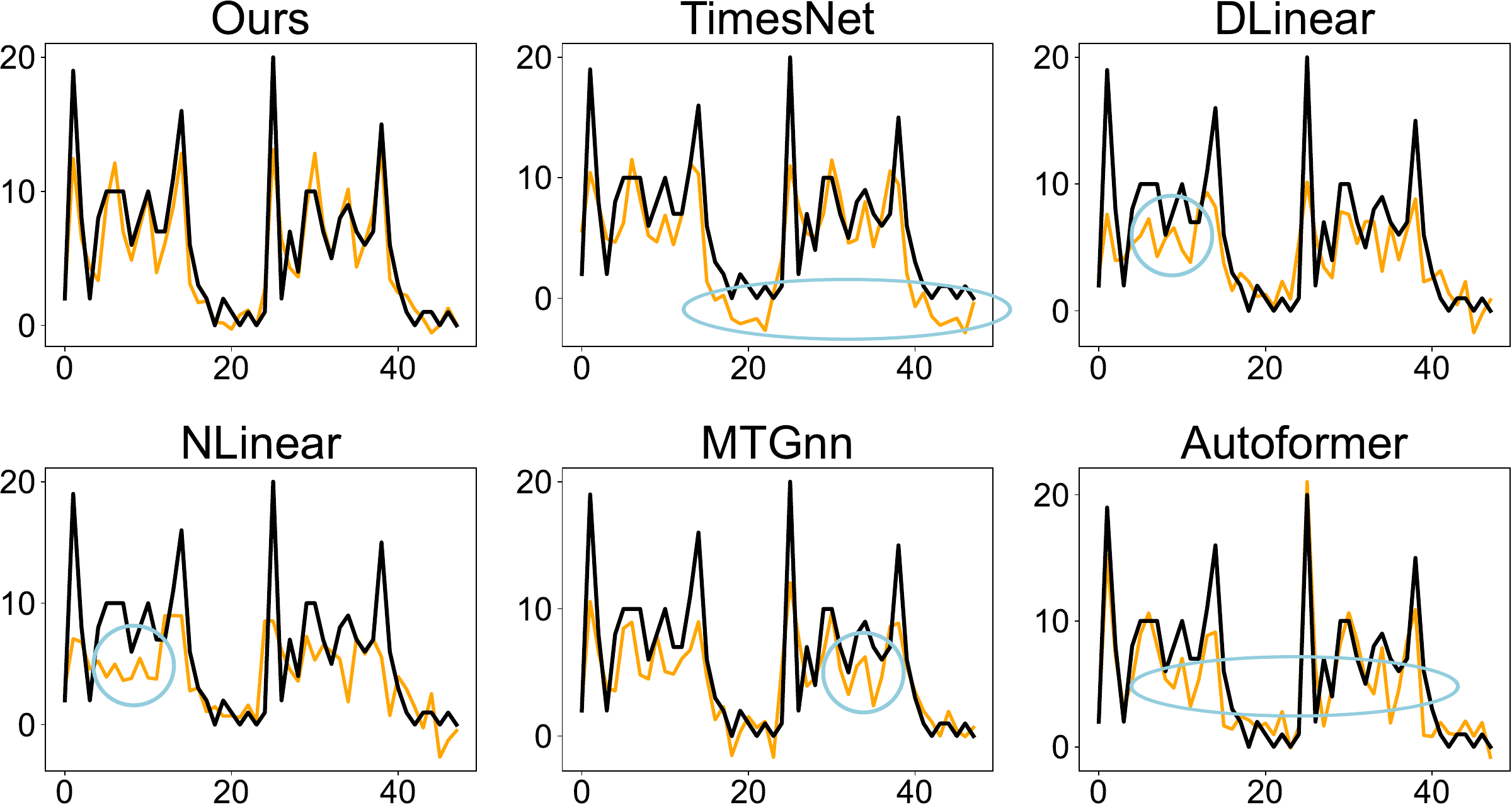}}
\caption{Visualization of Flight prediction results: black lines for true values, orange lines for predicted values, and blue markings indicating significant deviations.}
\label{fig: visualization}
\end{figure}

\subsection{Results and Analysis}
Table~\ref{tab: tab1} summarizes the predictive performance of all methods on 8 datasets, showcasing MSGNet's excellent results. Specifically, regarding the average Mean Squared Error (MSE) with different prediction lengths, it achieved the best performance on 5 datasets and the second-best performance on 2 datasets. In the case of the Flight dataset, MSGNet outperformed TimesNet (current SOTA), reducing MSE and MAE by $21.5\%$ (from 0.265 to 0.208) and $13.7\%$ (from 0.372 to 0.321) in average, respectively. Although TimesNet uses multi-scale information, it adopts a pure computer vision model to capture inter and intra-series correlations, which is not very effective for time series data. Autoformer demonstrated outstanding performance on the Flight dataset, likely attributed to its established autocorrelation mechanism. Nevertheless, even with GNN-based inter-series correlation modeling, MTGnn remained significantly weaker than our model due to a lack of attention to different scales. Furthermore, we assessed the model's generalization ability by calculating its average rank across all datasets. Remarkably, MSGNet outperforms other models on average ranking.

% \begin{figure}[htb!]
% \centerline{\includegraphics[scale=0.20]{cases.pdf}}
% \caption{Visualization of Flight prediction results: black lines for true values, orange lines for predicted values, and blue markings indicating significant deviations.}
% \label{fig: visualization}
% \end{figure}

MSGNet's excellence is evident in Figure~\ref{fig: visualization}, as it closely mirrors the ground truth, while other models suffer pronounced performance dips during specific time periods. The depicted peaks and troughs in the figure align with crucial flight data events, trends, or periodic dynamics. The inability of other models to accurately follow these variations likely stems from architecture constraints, hindering their capacity to grasp multi-scale patterns, sudden shifts, or intricate inter-series and intra-series correlations. %In contrast, MSGNet's advanced architecture adeptly captures the intricate flight data dynamics, showcasing its superior adaptability.

\subsection{Visualization of Learned Inter-series Correlation}
%MSGNet captures intra and inter-series dependencies within the scale, respectively. On the Flight dataset, for experiments predicting 96 time steps, the model frequently obtained 4 different periods, namely 24, 12, 6, and 4, corresponding to different long and short term patterns of 1 day, half day, morning, etc. Due to the data-driven nature, the time scale of model learning is not limited to this, but also includes long-term patterns such as 32, 48, and 96. This pattern of long and short cycles is similar to the time pattern of flights in practice, with a strong periodic property 24. It is also consistent with the visual results intuitively perceived from Figure XXX, proving that the time dependence learned by the model is close to reality. As the input sequence increases, the model can learn more complex and suitable time patterns from the sequence, thereby improving the effectiveness of the model. Please refer to XXX for specific content.

%As shown in Figure XXX, the model obtained the most significant k frequencies of Flight and other three datasets under the corresponding parameters on the test set, and displayed their corresponding cycle distribution proportions. It can also be inferred that time series in the real world often have complex time dependencies and multiple periodic properties.

Figure~\ref{fig: adj} illustrates three learned adjacency matrices for distinct time scales. In this instance, our model identifies three significant scales, corresponding to 24, 6, and 4 hours, respectively. As depicted in this showcase, our model learns different adaptive adjacency matrices for various scales, effectively capturing the interactions between airports in the flight data set. For instance, in the case of Airport 6, which is positioned at a greater distance from Airports 0, 1, and 3, it exerts a substantial influence on these three airports primarily over an extended time scale (24 hours). However, the impact diminishes notably as the adjacency matrix values decrease during subsequent shorter periods (6 and 4 hours). On the other hand, airports 0, 3, and 5, which are closer in distance, exhibit stronger mutual influence at shorter time scales. These observations mirror real-life scenarios, indicating that there might be stronger spatial correlations between flights at certain time scales, linked to their physical proximity. %With MSGNet, we can effectively subdivide these correlations into specific time scales, leading to more meaningful insights and accurate predictions.

\begin{figure}[htb!]

\centerline{\includegraphics[scale=0.37]{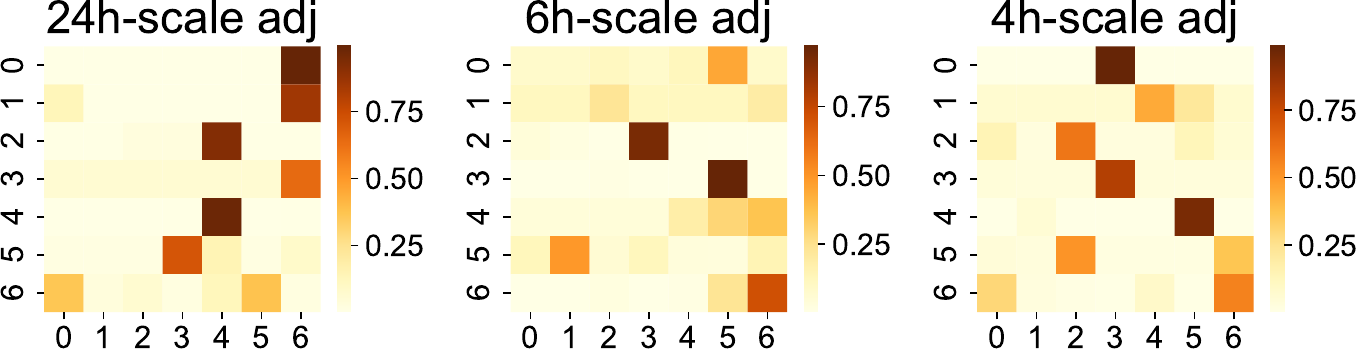}}
\centerline{\includegraphics[ width=8.1cm, height=4.1cm]{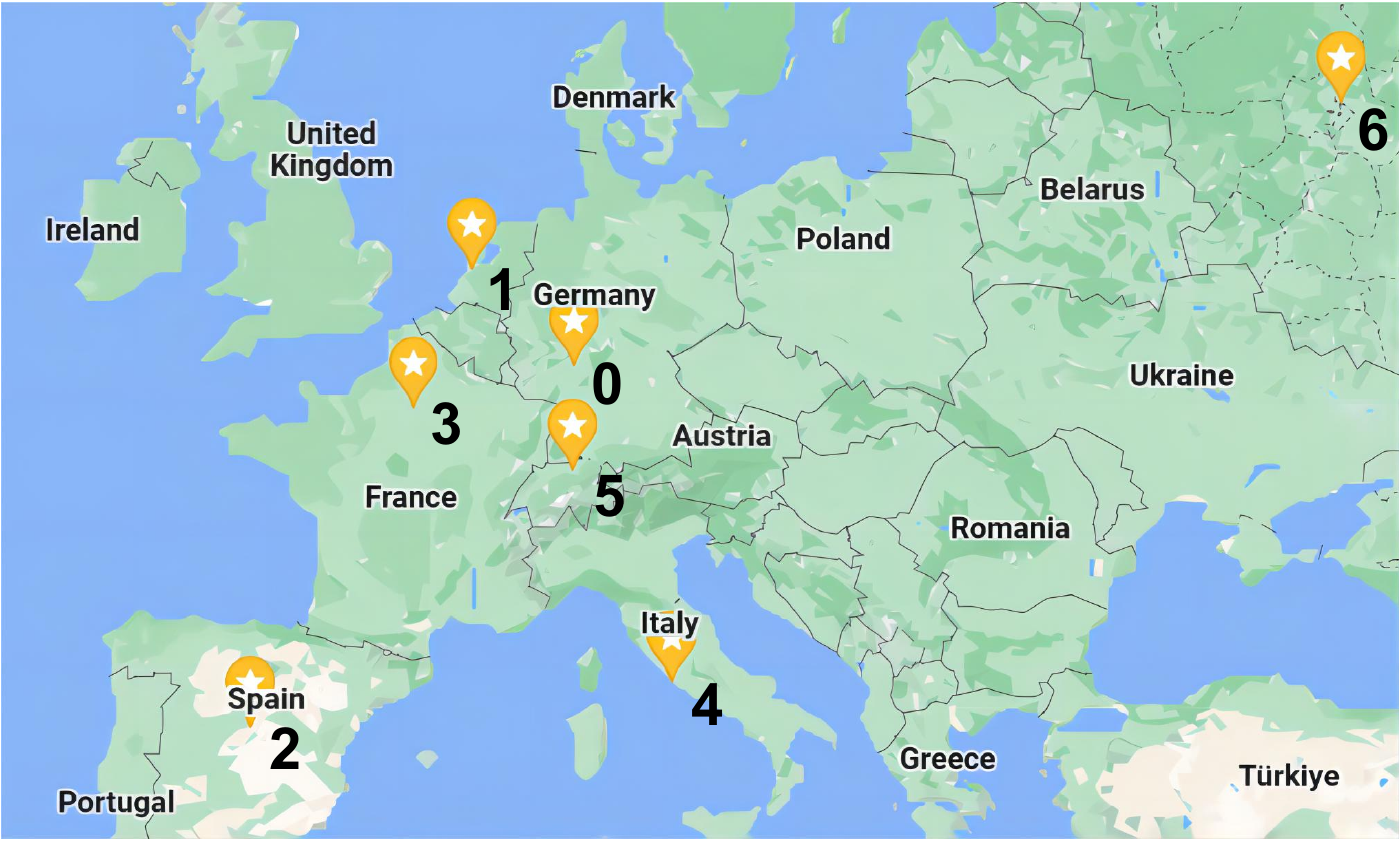}}
\caption{Learned adjacency matrices (24h, 6h, and 4h of the first layer) and airport map for Flight dataset.}
\label{fig: adj}
\end{figure}

\begin{table}[htb!]
\centerline{
% \resizebox{1\linewidth}{!}{
\small
\tabcolsep=0.18cm
\renewcommand\arraystretch{0.98}
\begin{tabular}{ccccccc}
\toprule
\multicolumn{1}{c}{Dataset}                                               & \multicolumn{2}{c}{Flight}                                                &
\multicolumn{2}{c}{Weather}                                               &  
\multicolumn{2}{c}{ETTm2}                                               \\ \midrule
\multicolumn{1}{c}{Metric}    & MSE     & MAE    & MSE    & MAE     & MSE    & MAE     \\ 
                      \toprule
\multicolumn{1}{c|}{MSGNet}  & {\textbf{0.195}} & {\textbf{0.311}} & {\textbf{0.218}} & {\textbf{0.255}} & {\textbf{0.245}} & {\textbf{0.304}} \\
\multicolumn{1}{c|}{w/o-AdapG}   & 0.302      & 0.401  & 0.232 & 0.270   & 0.253    & 0.313                       \\
\multicolumn{1}{c|}{w/o-MG}  & 0.213   & 0.331  & 0.226 & 0.261   & 0.250    & 0.307                                 \\
\multicolumn{1}{c|}{w/o-A}   & 0.198   & 0.314  & 0.224 & 0.259     & 0.247     & 0.306                                 \\
\multicolumn{1}{c|}{w/o-Mix} & 0.202     & 0.318  &  0.224&    0.260& 0.247       & 0.304                                 \\
\multicolumn{1}{c|}{TimesNet}   & 0.263    & 0.372    & 0.226 & 0.263 & 0.254   & 0.309                                         \\ \bottomrule
\end{tabular}
% }
}
 \caption{Ablation analysis of Flight, Weather and ETTm2 datasets. Results represent the average error of prediction length $\{96,336\}$, with the best performance highlighted in bold  black. }
 \label{tab: ablation}
\end{table}

\begin{table*}[h]
\centerline{
\resizebox{0.90\linewidth}{!}{
\begin{tabular}{cl|cccccccccccc}
\toprule
\multicolumn{2}{c}{Models}        & \multicolumn{2}{c}{\textbf{Ours}}                                             & \multicolumn{2}{c}{TimesNet} & \multicolumn{2}{c}{DLinear} & \multicolumn{2}{c}{NLinear} & \multicolumn{2}{c}{MTGnn} & \multicolumn{2}{c}{Autoformer} \\ \midrule
\multicolumn{2}{c}{Metric}        & MSE                                   & MAE                                   & MSE           & MAE          & MSE          & MAE          & MSE          & MAE          & MSE         & MAE         & MSE            & MAE           \\ \toprule
\multicolumn{2}{c|}{Flight(7:1:2)} & { \textbf{0.208}} & { \textbf{0.321}} & 0.265         & 0.372        & 0.233        & 0.345        & 0.285        & 0.388        & 0.280       & 0.378       & 0.238          & 0.344         \\
\multicolumn{2}{c|}{Flight(4:4:2)} & { \textbf{0.252}} & { \textbf{0.366}} & 0.335         & 0.426        & 0.332        & 0.448        & 0.365        & 0.447        & 0.407       & 0.501       & 0.307          & 0.424         \\
\multicolumn{2}{c|}{Decrease(\%)}  & { \textbf{21.29}} & { \textbf{13.80}} & 26.47         & 14.32        & 42.29        & 29.87        & 28.19        & 15.17        & 45.74       & 32.52       & 29.17          & 23.09         \\ \bottomrule
\end{tabular}
}
}
 \caption{Generalization test under COVID-19 influence: mean error for all prediction lengths, black bold indicates best performance. $\mathbf{Decrease}$ shows the percentage of performance reduction after partition modification.}
 \label{tab:covid}
\end{table*}

\subsection{Ablation Analysis}

We conducted ablation testing to verify the effectiveness of the MSGNet design. We considered 5 ablation methods and evaluated them on 3 datasets. The following will explain the variants of its implementation:

\begin{enumerate}
\item \textbf{w/o-AdapG:} We removed the adaptive graph convolutional layer (graph learning) from the model.
\item \textbf{w/o-MG:} We removed multi-scale graph convolution and only used a shared graph convolution layer to learn the overall inter-series dependencies.
\item \textbf{w/o-A:} We removed multi-head self-attention and eliminated intra-series correlation learning.
\item \textbf{w/o-Mix:} We replaced the mixed hop convolution method with the traditional convolution method~\cite{kipf2016semi}.
\end{enumerate}

Table \ref{tab: ablation} shows the results of the ablation study. Specifically, we have summarized the following four improvements:

    \begin{enumerate}
    \item \textbf{Improvement of graph learning layer: } After removing the graph structure, the performance of the model showed a significant decrease. This indicates that learning the inter-series correlation between variables is crucial in predicting multivariate time series.

    \item \textbf{Improvement of multi-scale graph learning: } Based on the results of the variant w/o-MG, it can be concluded that the multi-scale graph learning method significantly contributes to improving model performance. This finding suggests that there exist varying inter-series correlations among different time series at different scales.
    
    \item \textbf{Improvement of MHA layer: } Examining the results from w/o-A and TimesNet, it becomes apparent that employing multi-head self-attention yields marginal enhancements in performance.  %The pursuit of an ideal framework for intra-series correlation modeling persists as an ongoing and formidable challenge.  %Additionally, the model in this paper demonstrates clear advantages. When data needs to be embedded into a higher dimension for improved results, especially when the original data has a larger dimension, it can learn faster. Specific results can be found in Appendix Table XXX.

    \item \textbf{Improvement of mix-hop convolution: } The results of variant w/o-Mix indicate that the mix-hop convolution method is effective in improving the model's performance as w/o-Mix is slightly worse than MSGNet.
      \end{enumerate}

\subsection{Generalization Capabilities}

To verify the impact of the epidemic on flight predictions and the performance of MSGNet in resisting external influences, we designed a new ablation test by modifying the partitioning of the Flight dataset to 4:4:2. This design preserved the same test set while limiting the training set to data before the outbreak of the epidemic, and using subsequent data as validation and testing sets. The specific results are shown in Table~\ref{tab:covid}. By capturing multi-scale inter-series correlations, MSGNet not only achieved the best performance under two different data partitions but also exhibited the least performance degradation and strongest resistance to external influences. The results demonstrate that MSGNet possesses a robust generalization capability to out-of-distribution (OOD) samples. We hypothesize that this strength is attributed to MSGNet's ability to capture multiple inter-series correlations, some of which continue to be effective even under OOD samples of multivariate time series. This hypothesis is further supported by the performance of TimesNet, which exhibits a relatively small performance drop, ranking second after our method. It is worth noting that TimesNet also utilizes multi-scale information, similar to our approach.

%Normalize the training set, validation set, and test set using the mean and variance of the training data for two different partitions, Figure XXX shows a histogram of the distribution of the three. From the left figure, it can be seen that when considering epidemic factors in the training set, although the impact of the epidemic has significantly increased the distribution of low values in the training set, the distribution of the three is relatively similar, and each model can achieve good prediction results. From the right figure, it can be seen that when the training set does not consider epidemic factors, there are significant distribution changes among the three, which greatly affects the performance of the model. 不如放中文论文里的图

\section{Conclusion}

In this paper, we introduced MSGNet, a novel framework designed to address the limitations of existing deep learning models in time series analysis. Our approach leverages periodicity as the time scale source to capture diverse inter-series correlations across different time scales. Through extensive experiments on various real-world datasets, we demonstrated that MSGNet outperforms existing models in forecasting accuracy and captures intricate interdependencies among multiple time series. Our findings underscore the importance of discerning the varying inter-series correlation of different time scales in the analysis of time series data. %While MSGNet represents a promising step towards improved time series forecasting, there are several avenues for future research and development. Firstly, we can explore different methods for detecting and selecting prominent time scales. Furthermore, delving into advanced scale detection strategies and refining our modeling structures holds the potential to offer valuable insights for enhancing forecasting accuracy.

\section*{Acknowledgements}

This work was supported by the Natural Science Foundation of Sichuan Province (No. 2023NSFSC1423), the Fundamental Research Funds for the Central Universities, the open fund of state key laboratory of public big data (No. PBD2023-09), and the National Natural Science Foundation of China (No. 62206192). We also acknowledge the generous contributions of dataset donors.

\section{Appendix}

\section{1\quad A Mixture-of-Experts Perspective of MSGNet}

\subsection{1.1\quad Background: Mixture of Experts}

Mixture of experts is a well-established technique in the field of ensemble learning`\cite{jacobs1991adaptive}. It simultaneously trains a collection of expert models, denoted as ${f}_{i = 1, \cdots, k}$, which are designed to specialize in different input cases. The outputs generated by these experts are combined using a linear combination, where a "gating function" $g = [g_1, \ldots, g_k]$ determines the relative importance of each expert in the final decision-making process:
\begin{equation}
\text{MoE}(\boldsymbol{x}) = \sum_{i=1}^k g_i(\boldsymbol{x}) \cdot f_i(\boldsymbol{x}). 
\label{eq:MoE}
\end{equation}
The gating function, commonly implemented as a neural network, parameterizes the contribution of each expert. 

\subsection{1.2\quad Multi-Scale Graph Convolution: a Mixture-of-Experts Perspective}

For simplicity, we present a simplified form of our multi-scale graph convolution. In each layer, given the input $\boldsymbol{X} \in \mathbb{R}^{N \times c}$, we compute the transformed features as follows:
\begin{equation}
\hat{\boldsymbol{H}}_i = \boldsymbol{\hat{A}}^i \boldsymbol{X} \boldsymbol{W}_i,
\end{equation}
where $\hat{\boldsymbol{H}}_i \in \mathbb{R}^{N \times d}$ represents the $i$-th set of features, $\boldsymbol{\hat{A}}^i \in \mathbb{R}^{N \times N}$ corresponds to the $i$-th adjacency matrix, and $\boldsymbol{W}_i \in \mathbb{R}^{c \times d}$ denotes the learned transformation matrix.

Ignoring other operations in the ScaleGraph block, the output features of a ScaleGraph block are given by:
\begin{equation}
\boldsymbol{Z} \triangleq \text{ScaleGraphBlock}(\boldsymbol{X}) = \sum^k_{i=1} \hat{a}_i \hat{\boldsymbol{H}}_i,
\end{equation}
where $k$ represents the number of graph convolutions (scales). $\hat{a}_i$ serves as a gating function, similar to $g_i$ in Equation~\ref{eq:MoE}, and $\hat{\boldsymbol{H}}_i$ corresponds to the expert $f_i(\boldsymbol{X})$. It should be noted that $\hat{a}_i$ is also dependent on $\boldsymbol{X}$ since it is computed based on the amplitude of the time series' Fourier transformation.

If we set $k = 1$, it is evident that the model with only one graph convolution simplifies to $\boldsymbol{Z} = \boldsymbol{H}_1$, representing a single expert model. Numerous theoretical studies, such as those discussed in ~\cite{chen2022towards}, provide evidence that Mixture of Experts (MoE) outperforms single expert models. These studies highlight the advantages of leveraging multiple experts to enhance the model's capability in capturing complex patterns, leveraging diverse specialized knowledge, and achieving superior performance compared to a single expert approach.

\section{2\quad Representation Power Analysis}

We follow~\citet{abu2019mixhop} to analyze the representation power of different time series forecasting models. Firstly, we assume that there exists inter-series correlation between different time series, and the multivariate time series can be represented as a graph signal located on the graph $\mathcal{G} = \{ \mathcal{V}, \mathcal{E} \}$, where $\mathcal{V}$ is a set of nodes with $|\mathcal{V}| = N$, representing the number of time series, and $\mathcal{E} \subseteq \mathcal{V} \times \mathcal{V}$ is a set of edges. We can define the adjacency matrix $\boldsymbol{\hat{A}} \in \mathbb{R}^{N \times N}$ to represent the correlation between the $N$ time series. The adjacency matrix may be unknown; however, the model is expected to learn the features on the graph.

Similar to \cite{abu2019mixhop}, we analyze whether the model can learn the \textit{Two-hop Delta Operator} feature:

\begin{definition}
Representing Two-hop Delta Operator: A model is capable of representing a two-hop Delta Operator if there exists a setting of its parameters and an injective mapping $f$, such that the output of the network becomes
\begin{equation}
f\left(\sigma\left(\boldsymbol{\hat{A}X}\right) - \sigma\left(\boldsymbol{\hat{A}}^2\boldsymbol{X}\right)\right),
\end{equation}
given any adjacency matrix $\boldsymbol{\hat{A}}$, features $\boldsymbol{X}$, and activation function $\sigma$.
\end{definition}

The majority of time series prediction methods primarily concentrate on capturing the intra-series correlation of time series. Typical models, such as CNN~\cite{wu2023timesnet}, RNN~\cite{salinas2020deepar}, and transformer~\cite{zhou2021informer,wu2021autoformer} architectures, are employed to capture this correlation within individual time series. Meanwhile, the values of distinct time series at a given time instance are regarded as temporal features. These features are commonly transformed using MLPs, mapping them to a different space. Thus, when examining the feature dimension, various time series modeling methods essentially differ in terms of their utilization of distinct MLP parameter sharing strategies. 
From the time series variable dimension perspective, the output of an $l$-layer model without graph modeling can be represented as:
\begin{equation}
\sigma\left(\boldsymbol{W}^{(l-1)}(\sigma(\boldsymbol{W}^{(l-2)}\cdot \cdot \cdot \sigma(W^{(0)}\boldsymbol{X})\right),
\label{eq:nograph}
\end{equation}
where $\boldsymbol{W}^{(i)} \in \mathbb{R}^{N^{(i)} \times N^{(i+1)}}$ is a trainable weight matrix, and $\sigma$ denotes an element-wise activation function.

\begin{theorem}
The model defined by Equation~\ref{eq:nograph} is not capable of representing two-hop Delta Operators.
\end{theorem}

\begin{proof}
For the simplicity of the proof, let's assume that $\forall i, N_i = N$. In a particular case when $\sigma(x) = x$ and $\boldsymbol{X} = \boldsymbol{I}_n$, Equation~\eqref{eq:nograph} reduces to $\boldsymbol{W}^*,$ where $\boldsymbol{W}^* = \boldsymbol{W}^{(0)}\boldsymbol{W}^{(1)} \cdot \cdot \cdot \boldsymbol{W}^{(l-1)}$. 

Suppose the network is capable of representing a two-hop Delta Operator. This implies the existence of an injective map $f$ and a value for $\boldsymbol{W}^*$ such that $\forall \boldsymbol{\hat{A}}, \boldsymbol{W}^* = f(\boldsymbol{\hat{A}} - \boldsymbol{\hat{A}}^2)$.
Setting $\boldsymbol{\hat{A}} = I_n$, we find that $\boldsymbol{W}^* = f(\boldsymbol{0})$.

Let $\boldsymbol{C}$ be an arbitrary normalized adjacency matrix with $\boldsymbol{C} - \boldsymbol{C}^2 \neq \boldsymbol{0}$, e.g., 
\begin{equation}
\boldsymbol{C} =
\begin{bmatrix}
0.5 & 0.5 & 0 \\
0 & 0.5 & 0.5 \\
0.5 & 0.5 & 0 \\
\end{bmatrix}
\end{equation}

$\mathbf{D} = \mathbf{C} - \mathbf{C}^2$ is given by:
\begin{equation}
\begin{bmatrix}
0.25 & 0 & -0.25 \\
-0.25 & 0 & 0.25 \\
0.25 & 0 & -0.25 \\
\end{bmatrix}.
\end{equation}
Setting $\boldsymbol{\hat{A}} = \mathbf{C}$, we get $\boldsymbol{W}^* = f(\boldsymbol{D})$. 

The function $f$ is said to be injective provided that for all $\boldsymbol{a}$ and $\boldsymbol{b}$, if $\boldsymbol{a} \neq \boldsymbol{b}$, then $f(\boldsymbol{a}) \neq f(\boldsymbol{b})$. We have $f(\boldsymbol{D}) = f(\boldsymbol{0})$, and $\boldsymbol{D} \neq \boldsymbol{0}$. Thus, $f$ cannot be injective, proving that
the model using Equation~\ref{eq:nograph} cannot represent two-hop Delta Operators.

\end{proof}

Based on our analysis, we have observed that time series forecasting methods lacking advanced graph modeling are constrained to learn a fixed inter-series correlation pattern. This limitation becomes apparent when the inter-series correlation pattern of the target sequence changes, resulting in diminished generalizability and an inability to capture crucial features such as the Two-hop Delta Operator feature.

In contrast, our proposed model, MSGNet, harnesses the power of the mixhop method to learn multiple graph structures at various scales, presenting two significant advantages. Firstly, mixhop inherently possesses the ability to learn diverse features, including the Two-hop Delta Operator feature and general layer-wise neighborhood mixing features~\cite{abu2019mixhop}, enabling a more comprehensive representation of the data. Secondly, in situations where time series experience external disturbances, only specific inter-series correlations at certain scales may undergo changes, while other correlations remain unaffected. The incorporation of more diverse inter-series correlations ensures that MSGNet maintains its generalization performance even on out-of-distribution samples.

\section{3\quad More Details on Experiments}

\subsection{3.1\quad datasets}
\begin{table*}[t]
\renewcommand{\arraystretch}{1.3} 
\centerline{
\resizebox{0.8\linewidth}{!}{
\begin{tabular}{c|ccccc}
\toprule
\multirow{1}{*}{Datasets} & \multirow{1}{*}{Nodes} & \multirow{1}{*}{Input Length} & \multirow{1}{*}{Output Length} & \multirow{1}{*}{Train / test / valid Size} & \multirow{1}{*}{Frequency} \\
                          \midrule
Flight                   & 7                      & 96                            & \{96, 192, 336, 720\}          & (18317, 2633, 5261)           & Hourly                     \\
Weather                  & 21                     & 96                            & \{96, 192, 336, 720\}          & (36792, 5271, 10540)          & 10 minutes                 \\
ETTm1                    & 7                      & 96                            & \{96, 192, 336, 720\}          & (34465, 11521, 11521)         & 15 minutes                 \\
ETTm2                    & 7                      & 96                            & \{96, 192, 336, 720\}          & (34465, 11521, 11521)         & 15 minutes                 \\
ETTh1                    & 7                      & 96                            & \{96, 192, 336, 720\}          & (8545, 2881, 2881)            & Hourly                     \\
ETTh2                    & 7                      & 96                            & \{96, 192, 336, 720\}          & (8545, 2881, 2881)            & Hourly                     \\
Electricity              & 321                    & 96                            & \{96, 192, 336, 720\}          & (18317, 2633, 5261)           & Hourly                     \\
Exchange                 & 8                      & 96                            & \{96, 192, 336, 720\}          & (5120, 665, 1422)             & Daily                      \\ \bottomrule
\end{tabular}
}
}
\caption{Description of all datasets.}
\label{tab: dataset}
\end{table*}

The dataset information used in our experiment is shown in Table \ref{tab: dataset}. For the Flight dataset, we obtained the original data from OpenSky\footnote{https://opensky-network.org/}, which includes crucial information such as flight numbers, departure and destination airports, departure time, landing time, and other important details. To create this dataset, we focused on the flight data changes at seven major airports in Europe, including airports such as EDDF and EHAM, covering the period from January 2019 to December 2021. Additionally, we also gathered flight data specifically related to COVID-19 (after 2020). 

\begin{figure}[h]

\centerline{\includegraphics[scale=0.355]{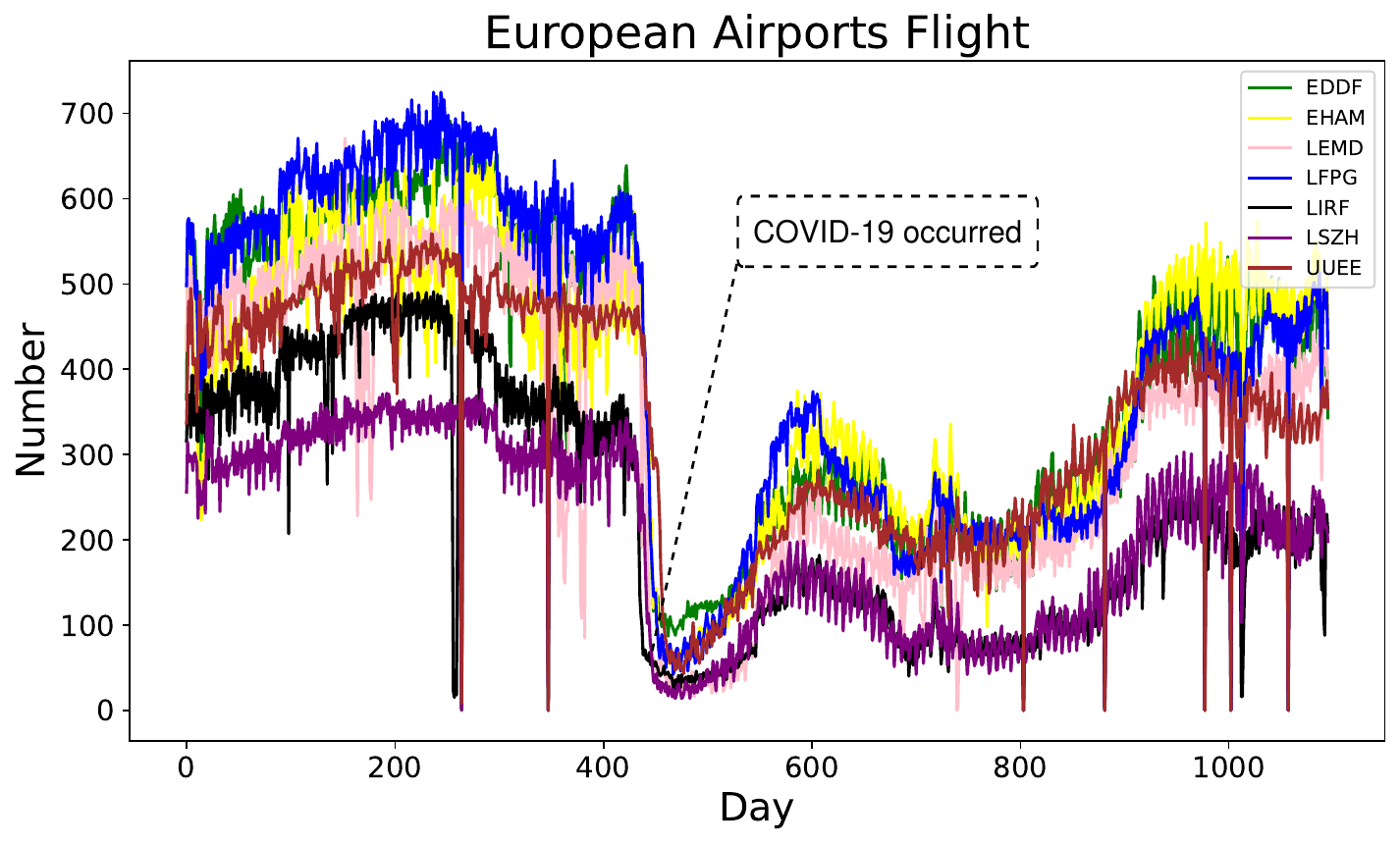}}
\caption{During the onset of the COVID-19 pandemic, there was a drastic decline in the daily flight volume at major airports in Europe, resembling a steep drop-off, which later experienced a gradual recovery.}
\label{fig: flight}
\end{figure}
\begin{figure}[ht]
\centerline{\includegraphics[scale=0.34]{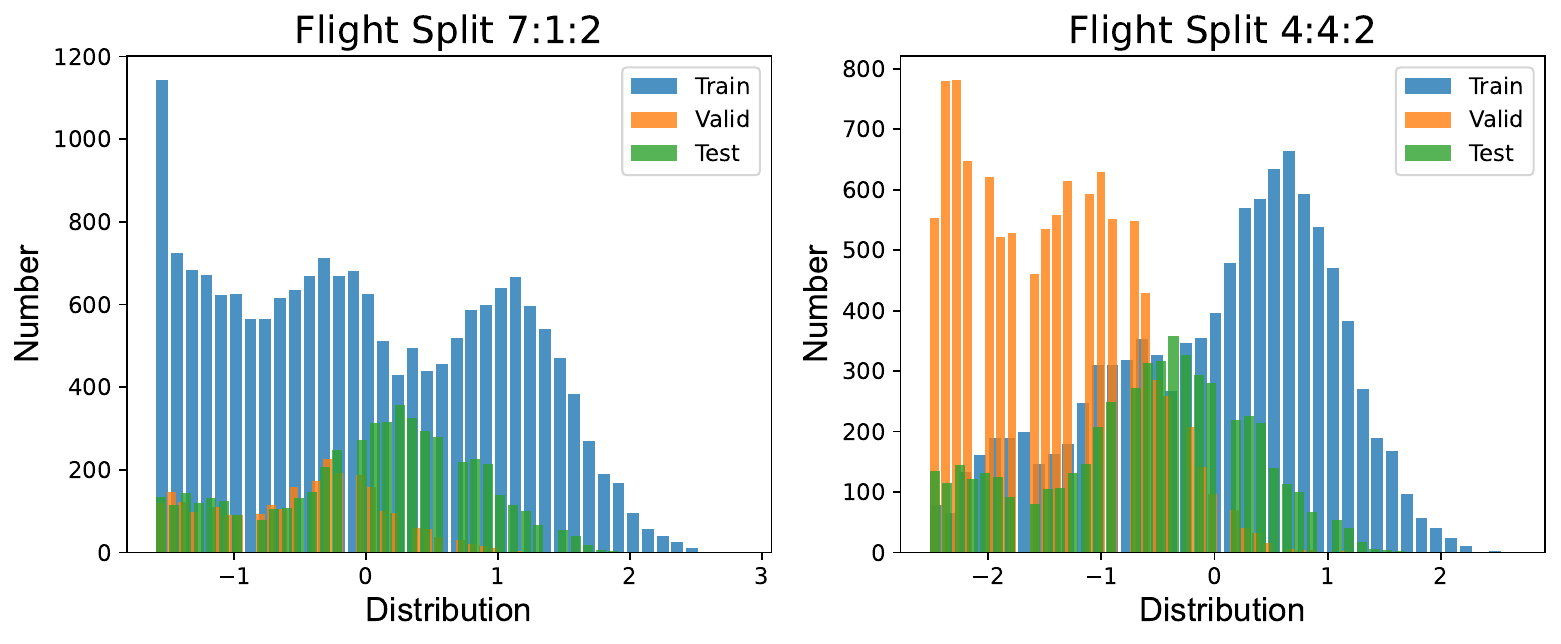}}

\caption{The distribution of data under two different partitions, with the vertical axis representing the number of data points.}
\label{fig: distribution}
\end{figure}

\begin{table}[t]
\renewcommand{\arraystretch}{1.2} 
\centerline{
% \resizebox{1\linewidth}{!}{
\begin{tabular}{c|clll}
\toprule
\multirow{1}{*}{Datasets} & \multicolumn{4}{c}{\multirow{1}{*}{Flight / Weather / ECL / Exchange}} \\
     \midrule
Epochs                   & \multicolumn{4}{c}{10}                                                         \\
Batch size               & \multicolumn{4}{c}{32}                                                         \\
Loss                     & \multicolumn{4}{c}{MSE}                                                        \\
Learning rate            & \multicolumn{4}{c}{1e-4}                                                       \\
k                        & \multicolumn{4}{c}{\{3, 5\}}                                              \\
Dim of $\mathbf{E}$               & \multicolumn{4}{c}{\{10, 30, 100\}}                                        \\
$d_{model}$                   & \multicolumn{4}{c}{\{16, 64, 1024\}}                                        \\
ScaleGraph block               & \multicolumn{4}{c}{2}                                                          \\
Mixhop order                   & \multicolumn{4}{c}{2}
                            \\
Heads            & \multicolumn{4}{c}{8}
                         \\
Optimizer                   & \multicolumn{4}{c}{Adam \cite{kingma2014adam}}
                    \\ \bottomrule
\end{tabular}
% }
}
\caption{Hyper-parameters on Flight, Weather, Electricity and Exchange.}
\label{tab: Hyper-Parameters1}
\end{table}

\begin{table}[t]
\renewcommand{\arraystretch}{1.2} 
\centerline{
\resizebox{1\linewidth}{!}{
\begin{tabular}{c|clll}
\toprule
\multirow{1}{*}{Datasets} & \multicolumn{4}{c}{\multirow{1}{*}{ETTm1 / ETTm2 / ETTh1/ ETTh2}} \\
                       \midrule
Epochs                   & \multicolumn{4}{c}{10}                                            \\
Batch size               & \multicolumn{4}{c}{32}                                            \\
Loss                     & \multicolumn{4}{c}{MSE}                                           \\
Learning rate            & \multicolumn{4}{c}{1e-4}                                          \\
k                        & \multicolumn{4}{c}{\{3, 5\}}                                 \\
Dim of $\mathbf{E}$                & \multicolumn{4}{c}{10}                                            \\
$d_{model}$                   & \multicolumn{4}{c}{\{16, 32\}}                             \\
ScaleGraph block               & \multicolumn{4}{c}{\{1, 2\}}                                 \\
Mixhop order                    & \multicolumn{4}{c}{2}                                             \\
Heads            & \multicolumn{4}{c}{8}
                         \\
Optimizer                   & \multicolumn{4}{c}{Adam \cite{kingma2014adam}}
                    \\ \bottomrule
\end{tabular}
}
}
\caption{Hyper-parameters on ETT.}
\label{tab: Hyper-Parameters2}
\end{table}

During the COVID-19 pandemic, our travel has been significantly impacted \cite{aktay2020google}. Naturally, air travel has also experienced substantial disruptions. This characteristic sets it apart from datasets like Weather and makes it suitable for assessing model stability with Out-of-Distribution (OOD) data. In Figure \ref{fig: flight}, we present a visual representation of the flight data changes for 7 major airports. To ensure clarity, we have used daily time granularity. As expected, the COVID-19 outbreak had a profound effect on flight operations.

For the Flight dataset, we conducted two types of partitioning: a 7:1:2 split and a 4:4:2 split, while keeping the same test set. In the second case, the training set does not include data after the outbreak of COVID-19. To ensure consistency, we normalized the training, validation, and test sets using the mean and variance of the training data, respectively. Figure~\ref{fig: distribution} illustrates the distribution histograms of the three sets. In the left graph, when considering COVID-19 factors in the training set, we observed a significant increase in the distribution of low values, reflecting the impact of the epidemic on the training data.  In general, the distributions among the three sets remained relatively similar. Conversely, as shown in the right graph, when the training set did not consider COVID-19 factors, there were notable distribution changes among the three sets.

\subsection{3.2\quad Hyper-Parameters}
We present the hyperparameters of MSGNet experiments on various datasets in Tables \ref{tab: Hyper-Parameters1} and \ref{tab: Hyper-Parameters2}, where $k$ represents the number of scales used. Dim of $E$ represents the dimension embedded in the node vector, taking a value in $\{10, 100\}$. Mixhop order is the depth of propagation in graph convolution.

We conducted an in-depth analysis of key hyperparameters within our model. Figure~\ref{fig: sensitivity} visually demonstrates the model's performance variations across distinct Mixhop orders and scale numbers ($k$), both ranging from 1 to 5. Our assessment encompasses the Flight, ETTh1, ETTh2, and Weather datasets, evaluating the Mean Squared Error (MSE) as the metric, with prediction lengths of $\{96, 192, 336, 720\}$.

Our proposed MSGNet exhibits consistent performance across a range of $k$ and Mixhop order selections. Notably, from Figure~\ref{fig: sensitivity}, we draw the following key observations:
\begin{itemize}
    \item In general, opting for a relatively smaller Mixhop order, such as 2, yields improved performance. This suggests that in each graph convolution layer, individual time series derive information solely from their 2-hop neighbors. This localized correlation structure benefits our model's predictions. In addition, for Flight, the overall impact of Mixhop order is small and can maintain stable performance. 
    \item Increasing the value of $k$ enhances the predictive performance. This effect can likely be attributed to the larger $k$ broadening the learned inter-series correlations, thereby promoting more diverse and informative predictions. Especially for datasets with multiple obvious scale patterns, increasing $k $ within a certain range can learn more detailed correlations, significantly improving model performance.
\end{itemize}

\begin{figure}[ht]
\centerline{\includegraphics[scale=0.32]{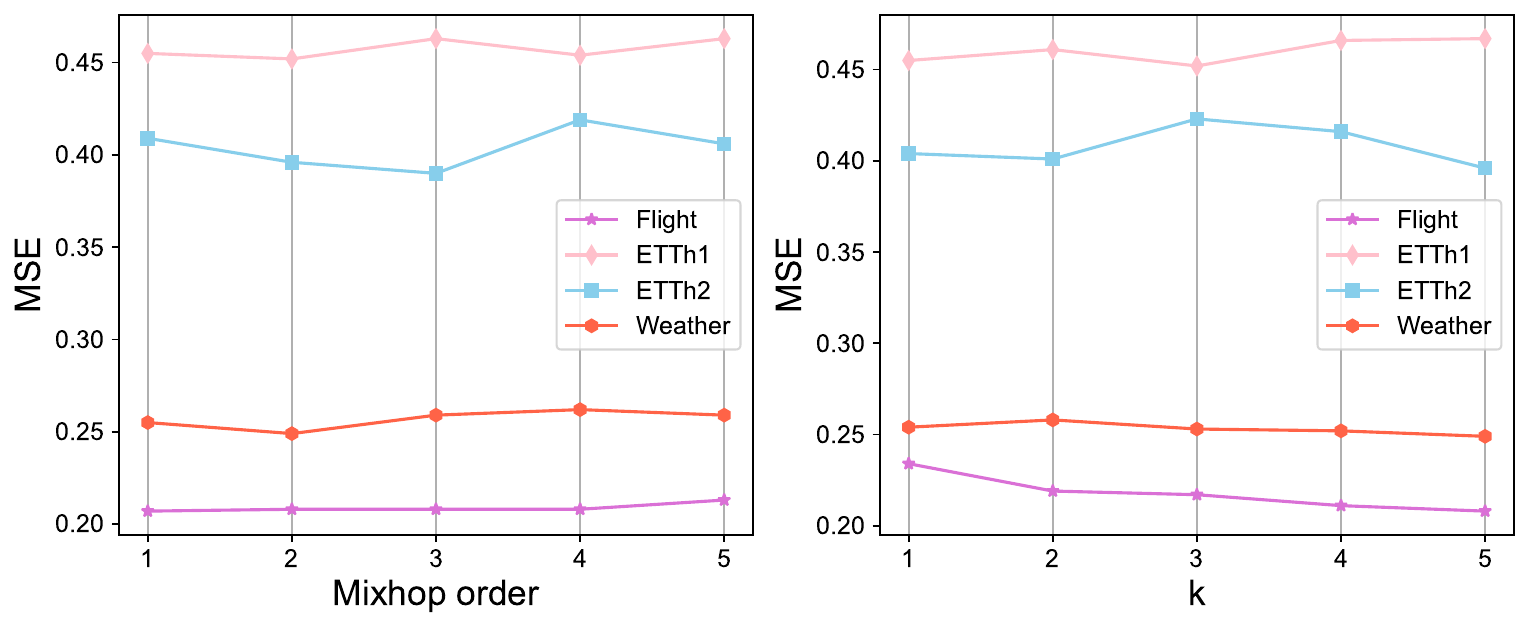}}

\caption{Sensitivity analysis of hyperparameters $k$ and Mixhop order on the Flight, ETTh1, ETTh2 and Weather datasets, showcasing the mean prediction error across different prediction lengths: $\{96, 192, 336, 720\}$.}
\label{fig: sensitivity}
\end{figure}

\begin{figure}[h]
\centerline{\includegraphics[scale=0.25]{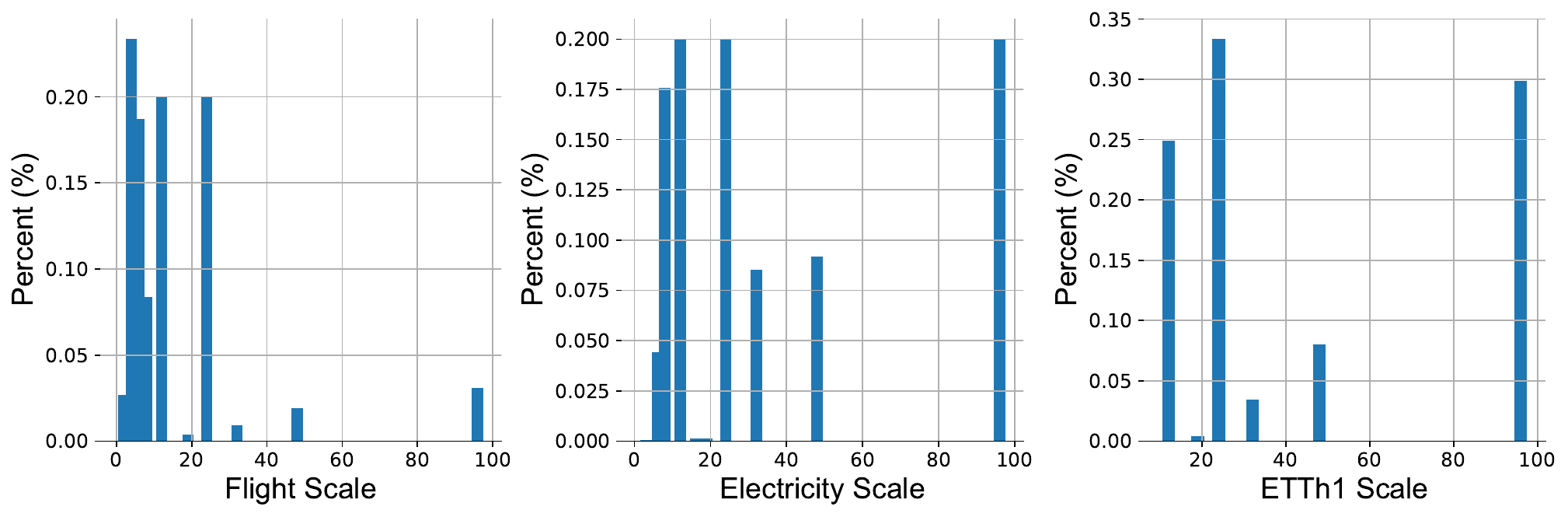}}

\caption{Scale distribution data of some datasets. Obtain $k$ scales each time through FFT, and normalize the proportion of different scales.}
\label{fig: scale}
\end{figure}

\section{4\quad The detected scale}

As depicted in Figure \ref{fig: scale}, MSGNet successfully identified the most significant $k$ frequencies across three datasets during testing. The corresponding scale distributions were also shown. For the Flight dataset, the model consistently captured multiple diverse scales while predicting the future $96$ time steps. These scales included various long and short-term patterns such as 1 day, half day, morning, and more. The observed patterns closely resemble real-world flight patterns, exhibiting strong scaling properties consistent with subjective visualization results. This finding validates that the model effectively learns time dependencies close to reality.

\section{5\quad Performances under Longer Input Sequences}

Generally, the size of the review window influences the types of dependencies the model can learn from historical information. A proficient time series forecasting model should be able to accurately capture dependencies over extended review windows, leading to improved results.

In a prior study~\cite{zeng2023transformers}, it was demonstrated that Transformer-based models tend to display noticeable fluctuations in their performance, leading to either a decline in overall performance or reduced stability with longer review window. These models often achieve their best or near-optimal results when the review window is set to $96$. On the other hand, linear models show a gradual improvement as the review window increases.

We also conducted a similar analysis on the Flight dataset, employing various review windows, namely $\{48, 72, 96, 120, 144, 168, 192, 336, 504, 672, 720\}$, to forecast the values of the subsequent $336$ time steps. The Mean Squared Error (MSE) served as the chosen error function for evaluation. The detailed results can be found in Figure~\ref{fig: diffseq}. 

MSGNet also incorporates a self-attention mechanism for extracting temporal dependencies. However, unlike previous models that might suffer from overfitting temporal noise, MSGNet excels in capturing temporal information effectively. While our model's performance is slightly inferior to the linear model under longer review window in this case, it demonstrates substantial improvement compared to other models. MSGNet overcomes issues of significant rebound and strong fluctuations, displaying an overall trend of decreasing error. This robust behavior showcases MSGNet's capability to reliably extract long sequence time dependencies. We think the reason behind this is MSGNet's transformer operating within the scale itself. Through scale transformation, it shortens long sequences into shorter ones, effectively compensating for the transformer's limitations in capturing long-term sequence correlations of time series. For instance, in a sequence with a length of 720, once a period scale of 24 is identified, it gets reshaped into a scale tensor of 24 $\times$ 30. The transformer then operates on this scale of length 24, instead of the length of 720.

Furthermore, we present a deeper analysis of MSGNet's performance on the ETT (h1, h2, m1, m2) dataset using various review windows in Figure \ref{fig: longerseq}. This aims to validate the efficacy of MSGNet when operating within extended review windows. Notably, it becomes evident that an extended review window yields enhancement in MSGNet's performance. This outcome attests to the role of scale transformation in mitigating challenges encountered by Transformers when dealing with inputs from a more extensive time horizon.

% \begin{figure}[ht]
%     \centering
%     \subfigure[Flight dataset predictions for 336 time steps  with different review windows. We use four other models for comparison.]{\includegraphics[scale=0.4]{diffseq.pdf}}
%     \subfigure[MSGnet's ETT dataset prediction performance for 336 time steps with different review windows.]
%     {\includegraphics[scale=0.4]{longerseq.pdf}}

%     \caption{Caption}
%     \label{fig:enter-label}
% \end{figure}

\begin{figure}[h]
\centerline{\includegraphics[scale=0.45]{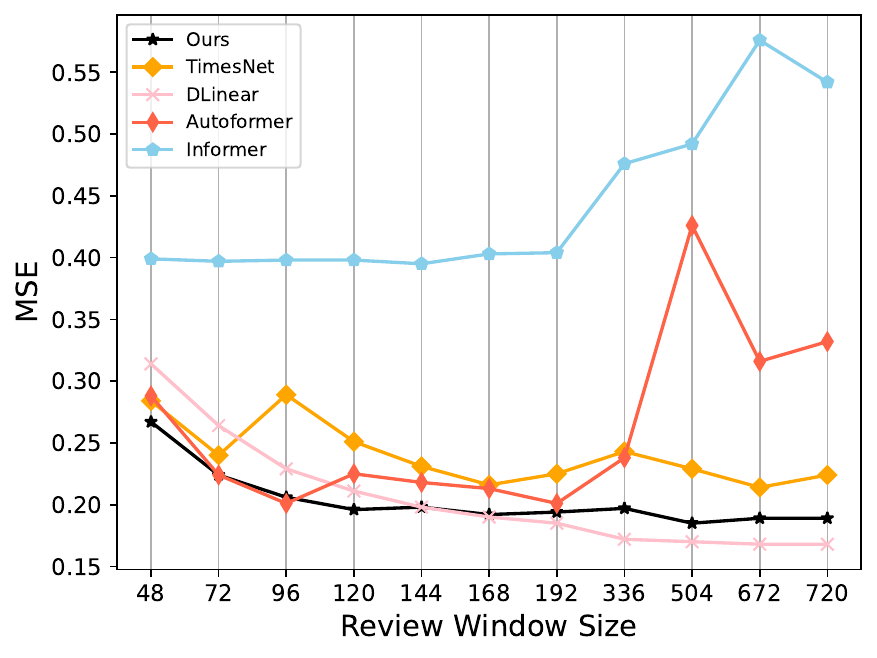}}
\caption{Flight dataset predictions for 336 time steps  with different review windows. We use four other models for comparison.
}
\label{fig: diffseq}
\end{figure}

\begin{figure}[h]
\centerline{\includegraphics[scale=0.45]{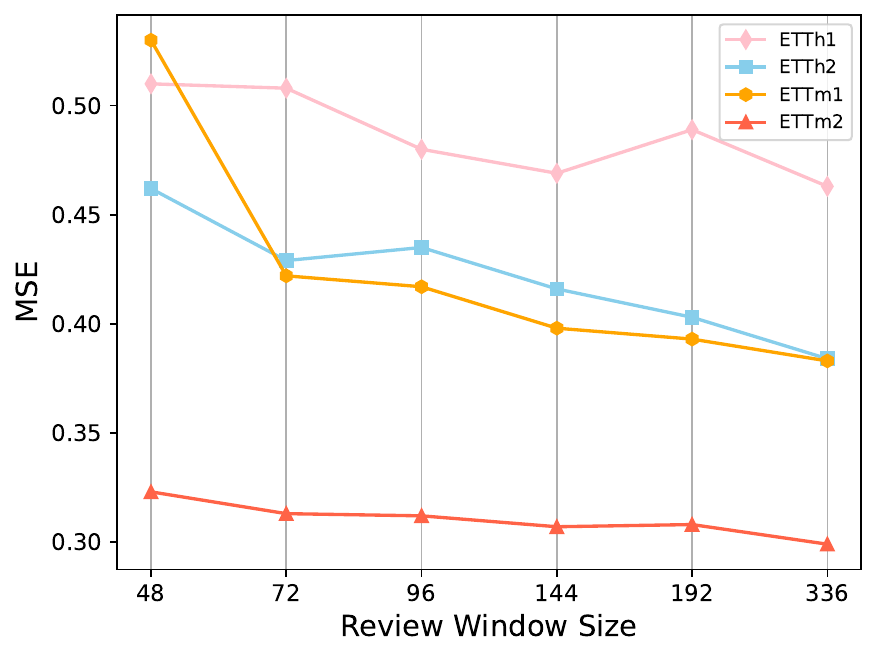}}
\caption{MSGnet's ETT dataset prediction performance for 336 time steps with different review windows.}
\label{fig: longerseq}
\end{figure}

\section{6\quad Computational Efficiency}

In terms of efficiency, we chose to evaluate the models on a more complex Electricity dataset to analyze GPU memory usage, running speed, and MSE ranking for various prediction lengths using different methods. This comprehensive approach enabled us to consider both efficiency and effectiveness thoroughly. To ensure fairness, all models were tested with a Batch size of 32, and the results can found in Table~\ref{tab: efficiency}. Importantly, our model has surpassed TimesNet in operational efficiency, substantially reducing training time while achieving similar time costs across different prediction lengths.

\begin{table}[h]
\renewcommand{\arraystretch}{1.1} 
\centerline{
\resizebox{\linewidth}{!}{
\begin{tabular}{ccclclcl}
\toprule
                                                  &                               & \multicolumn{2}{c}{}                                                                             & \multicolumn{2}{c}{}                                                                                   & \multicolumn{2}{c}{}                           \\
\multirow{-2}{*}{Models}                          & \multirow{-2}{*}{Pred Length} & \multicolumn{2}{c}{\multirow{-2}{*}{\begin{tabular}[c]{@{}c@{}}GPU Memory \\ (GB)\end{tabular}}} & \multicolumn{2}{c}{\multirow{-2}{*}{\begin{tabular}[c]{@{}c@{}}Running Time\\ (s/epoch)\end{tabular}}} & \multicolumn{2}{c}{\multirow{-2}{*}{MSE Rank}} \\ \toprule
\multicolumn{1}{c|}{}                             & \multicolumn{1}{c|}{96}       & \multicolumn{2}{c}{{\color[HTML]{000000} 13.04}}                                                 & \multicolumn{2}{c}{{\color[HTML]{000000} 320.38}}                                                      & \multicolumn{2}{c}{{\textbf{1}}}   \\
\multicolumn{1}{c|}{}                             & \multicolumn{1}{c|}{192}      & \multicolumn{2}{c}{{\color[HTML]{000000} 13.05}}                                                 & \multicolumn{2}{c}{{\color[HTML]{000000} 316.46}}                                                      & \multicolumn{2}{c}{{\textbf{1}}}   \\
\multicolumn{1}{c|}{}                             & \multicolumn{1}{c|}{336}      & \multicolumn{2}{c}{{\color[HTML]{000000} 16.95}}                                                 & \multicolumn{2}{c}{{\color[HTML]{000000} 471.37}}                                                      & \multicolumn{2}{c}{{\textbf{1}}}   \\
\multicolumn{1}{c|}{\multirow{-4}{*}{Ours}}       & \multicolumn{1}{c|}{720}      & \multicolumn{2}{c}{{\color[HTML]{000000} 16.98}}                                                 & \multicolumn{2}{c}{{\color[HTML]{000000} 482.97}}                                                      & \multicolumn{2}{c}{2}                          \\ \midrule
\multicolumn{1}{c|}{}                             & \multicolumn{1}{c|}{96}       & \multicolumn{2}{c}{{\color[HTML]{000000} 8.81}}                                                  & \multicolumn{2}{c}{{\color[HTML]{000000} 748.79}}                                                      & \multicolumn{2}{c}{2}                          \\
\multicolumn{1}{c|}{}                             & \multicolumn{1}{c|}{192}      & \multicolumn{2}{c}{{\color[HTML]{000000} 10.07}}                                                 & \multicolumn{2}{c}{{\color[HTML]{000000} 821.39}}                                                      & \multicolumn{2}{c}{2}                          \\
\multicolumn{1}{c|}{}                             & \multicolumn{1}{c|}{336}      & \multicolumn{2}{c}{{\color[HTML]{000000} 14.48}}                                                 & \multicolumn{2}{c}{{\color[HTML]{000000} 2019.38}}                                                     & \multicolumn{2}{c}{2}                          \\
\multicolumn{1}{c|}{\multirow{-4}{*}{TimesNet}}   & \multicolumn{1}{c|}{720}      & \multicolumn{2}{c}{{\color[HTML]{000000} 20.22}}                                                 & \multicolumn{2}{c}{{\color[HTML]{000000} 3520.99}}                                                     & \multicolumn{2}{c}{{\textbf{1}}}   \\ \midrule
\multicolumn{1}{c|}{}                             & \multicolumn{1}{c|}{96}       & \multicolumn{2}{c}{{\color[HTML]{000000} 2.41}}                                                  & \multicolumn{2}{c}{{\color[HTML]{000000} 4.42}}                                                        & \multicolumn{2}{c}{3}                          \\
\multicolumn{1}{c|}{}                             & \multicolumn{1}{c|}{192}      & \multicolumn{2}{c}{{\color[HTML]{000000} 2.45}}                                                  & \multicolumn{2}{c}{{\color[HTML]{000000} 5.80}}                                                        & \multicolumn{2}{c}{3}                          \\
\multicolumn{1}{c|}{}                             & \multicolumn{1}{c|}{336}      & \multicolumn{2}{c}{{\color[HTML]{000000} 2.49}}                                                  & \multicolumn{2}{c}{{\color[HTML]{000000} 9.28}}                                                        & \multicolumn{2}{c}{3}                          \\
\multicolumn{1}{c|}{\multirow{-4}{*}{DLinear}}    & \multicolumn{1}{c|}{720}      & \multicolumn{2}{c}{{\color[HTML]{000000} 2.62}}                                                  & \multicolumn{2}{c}{{\color[HTML]{000000} 23.57}}                                                       & \multicolumn{2}{c}{3}                          \\ \midrule
\multicolumn{1}{c|}{}                             & \multicolumn{1}{c|}{96}       & \multicolumn{2}{c}{{\color[HTML]{000000} 4.27}}                                                  & \multicolumn{2}{c}{{\color[HTML]{000000} 54.66}}                                                       & \multicolumn{2}{c}{4}                          \\
\multicolumn{1}{c|}{}                             & \multicolumn{1}{c|}{192}      & \multicolumn{2}{c}{{\color[HTML]{000000} {5.25}}}                                         & \multicolumn{2}{c}{{\color[HTML]{000000} 66.89}}                                                       & \multicolumn{2}{c}{4}                          \\
\multicolumn{1}{c|}{}                             & \multicolumn{1}{c|}{336}      & \multicolumn{2}{c}{6.29}                                                                         & \multicolumn{2}{c}{88.19}                                                                              & \multicolumn{2}{c}{4}                          \\
\multicolumn{1}{c|}{\multirow{-4}{*}{Autoformer}} & \multicolumn{1}{c|}{720}      & \multicolumn{2}{c}{9.76}                                                                         & \multicolumn{2}{c}{143.17}                                                                             & \multicolumn{2}{c}{4}                          \\ \bottomrule
\end{tabular}
}
}
\caption{GPU memory, running time, and MSE rank of MSGNet, TimesNet, Dlinear, and Autoformer. }
\label{tab: efficiency}
\end{table}

This is perfectly normal because as the input time increases, MSGNet's MHA continues to operate solely on short time scales, with each scale sharing the same operation among the others. The graph convolution module is also influenced solely by the number of scales as a hyperparameter. Without increasing the number of scales, the computational complexity of these modules remains unchanged. In contrast, for TimesNet, it performs 2D convolutions in both the scale and the number of scales dimensions. Consequently, as the input time lengthens, the convolution operations will correspondingly increase.

It should be noted that our model is computationally heavier compared to two other simpler models, Dlinear and Autoformer. Dlinear is a straightforward linear model, so it's natural that it uses fewer GPU resources. As for Autoformer, we also observed a sharp increase in computation cost with longer input lengths. This is reasonable since its MHA operates on the entire sequence instead of just shorter time scales.

\section{7\quad More Showcases     }
We provide some showcases in Figures \ref{fig: Flight_96} and \ref{fig: ETTh2_336}. Compared to other models, it is obvious that MSGNet can better fit the trend changes and periodicity of data.

\begin{figure}[ht]
\centerline{\includegraphics[scale=0.19]{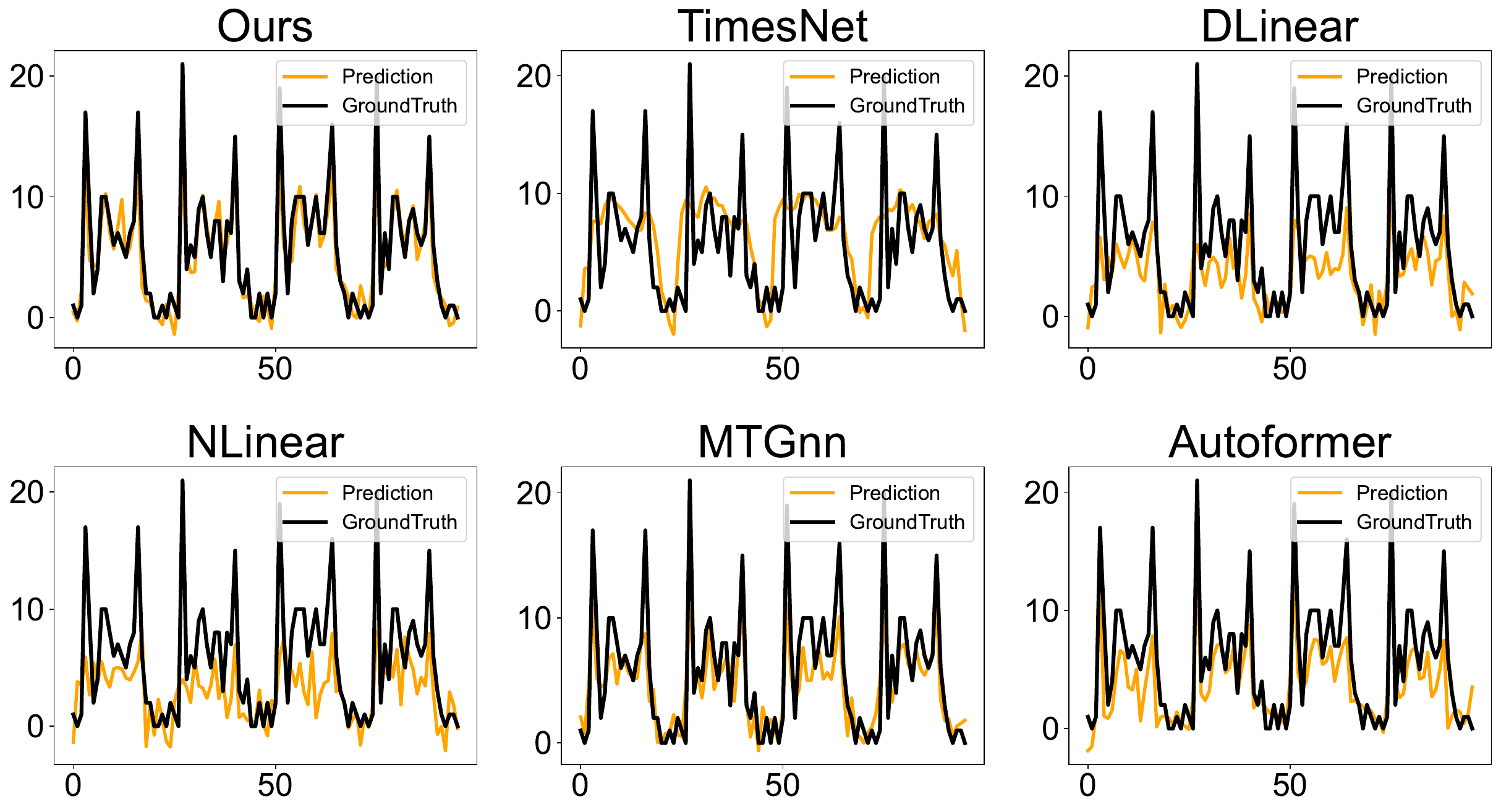}}

\caption{Visualize the Flight dataset prediction with input length $96$ and output length $96$ settings. The selected sequence id is 4.
}
\label{fig: Flight_96}
\end{figure}

\begin{figure}[ht]
\centerline{\includegraphics[scale=0.19]{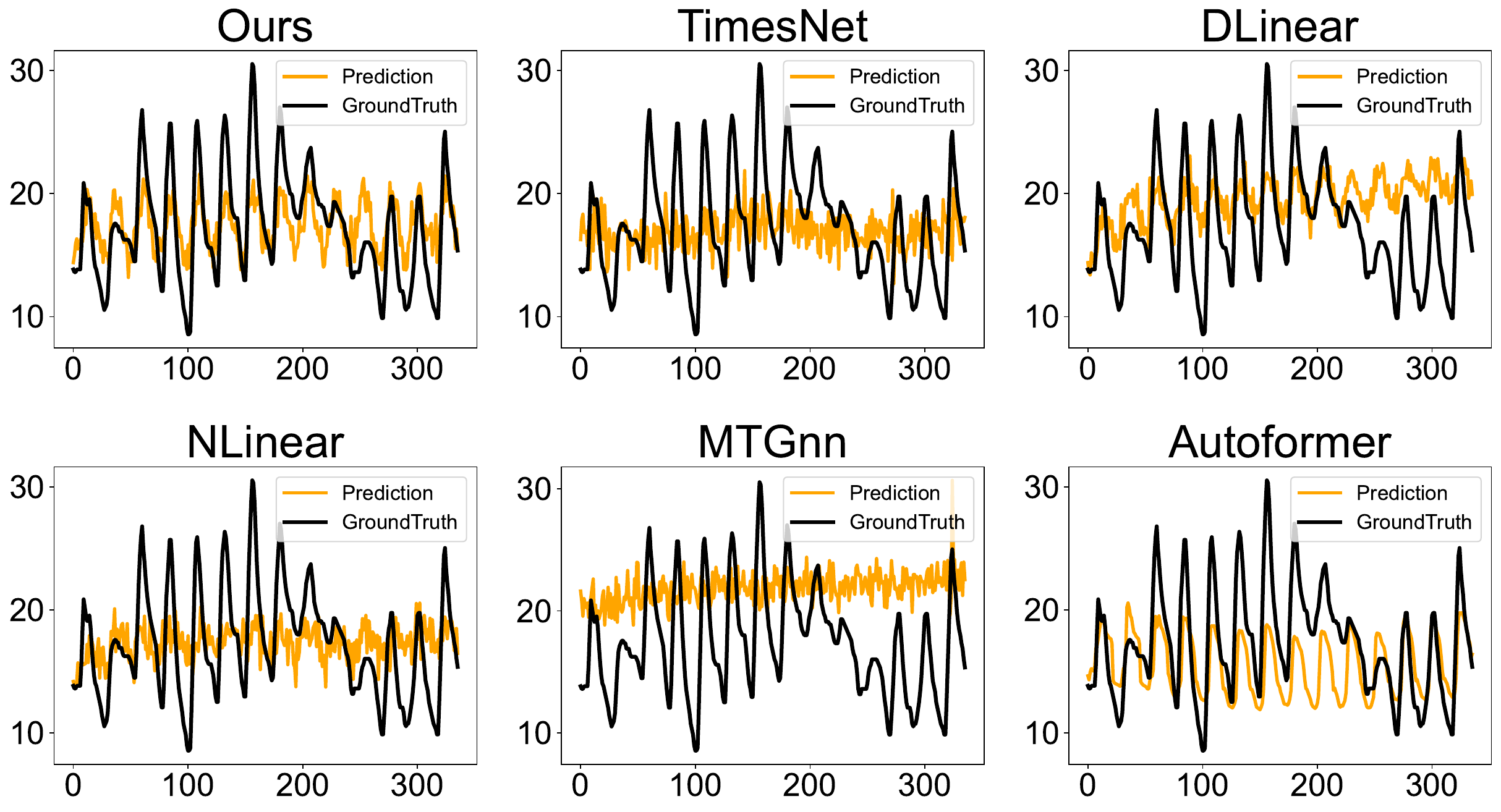}}

\caption{Visualize the prediction of the ETTm2 dataset with an input length of $96$ and an output length of $336$. The selected sequence id is 6.}
\label{fig: ETTh2_336}
\end{figure}

\bibliography{aaai24}

\end{document}